%% file: neurips_2026.tex
\documentclass{article}

\PassOptionsToPackage{numbers, compress}{natbib}

\usepackage[preprint]{neurips_2026}

\usepackage[utf8]{inputenc}
\usepackage[T1]{fontenc}

\usepackage{xcolor}

\definecolor{nipsblue}{rgb}{0.0, 0.0, 0.55}
\definecolor{nipsred}{rgb}{0.6, 0.0, 0.0}
\usepackage[rightcaption]{sidecap}

\usepackage{hyperref}
\hypersetup{
    colorlinks=true,
    linkcolor=nipsred,
    citecolor=nipsblue,
    urlcolor=nipsblue,
    pdftitle={NeurIPS 2025 Paper},
}

\usepackage[bottom]{footmisc}
\usepackage{url}
\usepackage{wrapfig}
\usepackage{booktabs}
\usepackage{amsfonts}
\usepackage{nicefrac}
\usepackage{microtype}
\usepackage[percent]{overpic}
\usepackage{graphicx}
\usepackage{amsmath}
\usepackage{mathrsfs}
\usepackage[export]{adjustbox}
\usepackage{amsthm}
\usepackage{subcaption}
\usepackage{algorithm}
\usepackage{algpseudocode}
\usepackage{upgreek}

\newcommand{\rmd}{\mathrm{d}}

\definecolor{nipsgraybg}{RGB}{248, 248, 248}
\definecolor{nipsgrayframe}{RGB}{180, 180, 180}

\definecolor{nipsmistbg}{RGB}{250, 250, 250}
\definecolor{nipsmistframe}{RGB}{200, 200, 200}

\definecolor{nipssagebg}{RGB}{250, 250, 250}
\definecolor{nipssageframe}{RGB}{200, 200, 200}

\usepackage[most]{tcolorbox}
\tcbuselibrary{theorems}

\newcounter{lemma}[section]

\newtcbtheorem[use counter=lemma]{lemma}{Lemma}
{
  colback=nipsmistbg,
  colframe=nipsmistframe,
  coltitle=black,
  fonttitle=\bfseries,
  sharp corners,
  boxrule=0.3mm,
  top=2mm, bottom=2mm, left=2mm, right=2mm,
  toptitle=0.5mm, bottomtitle=0.5mm,
  breakable
}{lem}

\newcounter{definition}[section]

\newtcbtheorem[use counter=definition]{definition}{Definition}
{
  colback=nipsmistbg,
  colframe=nipsmistframe,
  coltitle=black,
  fonttitle=\bfseries,
  sharp corners,
  boxrule=0.3mm,
  top=2mm, bottom=2mm, left=2mm, right=2mm,
  toptitle=0.5mm, bottomtitle=0.5mm,
  breakable
}{definition}

\newcounter{proposition}[section]

\newtcbtheorem[use counter=proposition]{proposition}{Proposition}
{
  colback=nipsmistbg,
  colframe=nipsmistframe,
  coltitle=black,
  fonttitle=\bfseries,
  sharp corners,
  boxrule=0.3mm,
  top=2mm, bottom=2mm, left=2mm, right=2mm,
  toptitle=0.5mm, bottomtitle=0.5mm,
  breakable
}{proposition}

\newcounter{theorem}[section]

\newtcbtheorem[use counter=theorem]{theorem}{Theorem}
{
  colback=nipsmistbg,
  colframe=nipsmistframe,
  coltitle=black,
  fonttitle=\bfseries,
  sharp corners,
  boxrule=0.3mm,
  top=2mm, bottom=2mm, left=2mm, right=2mm,
  toptitle=0.5mm, bottomtitle=0.5mm,
  breakable
}{thm}

\newtcolorbox{definitionbox}[1]{
  colback=nipsgraybg,
  colframe=nipsgrayframe,
  coltitle=black,
  fonttitle=\bfseries,
  sharp corners,
  boxrule=0.3mm,
  top=2mm, bottom=2mm, left=2mm, right=2mm,
  title={Definition: #1},
  breakable
}

\definecolor{nipsproofbg}{RGB}{250, 250, 250}
\definecolor{nipsproofframe}{RGB}{200, 200, 200}

\newtcolorbox{proofbox}[1][]{
  enhanced,
  breakable,
  sharp corners,
  colback=nipsproofbg,
  colframe=nipsproofframe,
  boxrule=0.18mm,
  top=2mm, bottom=2mm, left=2mm, right=2mm,
  before upper={
    \if\relax\detokenize{#1}\relax
    \else
      \textit{#1}\par\smallskip
    \fi
  },
  after upper={\par\hfill$\square$}
}

\title{AsyncPatch Diffusion: spatially-flexible image generation}

\author{%
  Samuele Papa$^{1,*}$ \quad Valentin De Bortoli$^1$ \quad Guillaume Couairon$^1$ \\ \textbf{Daniel Sýkora$^1$ \quad Romuald Elie$^1$ \quad Klaus Greff$^1$ }\\
  $^1$Google DeepMind%
  \vspace{-2em}%
}

\begin{document}

\maketitle
\renewcommand{\thefootnote}{\fnsymbol{footnote}}
\footnotetext[1]{Work done while Samuele was Student Researcher at Google DeepMind Berlin and a PhD student affiliated with the University of Amsterdam and The Netherlands Cancer Institute.}
\renewcommand{\thefootnote}{\arabic{footnote}}

\begin{abstract}

  Standard diffusion models corrupt an entire sample with a single shared noise level, forcing all spatial regions to follow the same denoising trajectory.
  We introduce \textbf{AsyncPatch Diffusion}, a joint-diffusion framework that assigns distinct noise levels to different input dimensions, such as image pixels, or latent tokens.
  We show how this asynchronous corruption defines a valid generative process while supporting a richer family of spatially heterogeneous denoising trajectories, and prove the first \emph{valid ELBO} for this process.
  We show that a single pretrained model can perform spatially adaptive generation, where different regions are denoised on different schedules. A key challenge is training: naive independent noise-level sampling overemphasizes highly heterogeneous configurations and underrepresents homogeneous noise levels, that are crucial during sampling.
  We address this with a controlled noise-level sampler that regulates both the average corruption level and its spatial variability.
  AsyncPatch achieves generation quality comparable to conventional diffusion on ImageNet 256 and LSUN, while being natively suited for inpainting without task-specific fine-tuning.
  We further introduce input guidance, which uses clean or partially corrupted regions to guide the generation of unknown regions, improving local consistency and texture matching.
  Finally, we demonstrate adaptive generation strategies including uncertainty-guided acceleration and autoregressive sampling.

\end{abstract}

\section{Introduction}
Diffusion models have emerged as a powerful class of generative model for complex modalities, including images~\citep{sohl2015deep,rombach2022high,peebles2023scalable}, videos~\citep{ho2022video, blattmann2023align}, audio~\citep{kong2021diffwave}, 3D scenes~\citep{poole2023dreamfusion}, text~\citep{xu_energy-based_2025}, and protein design~\citep{watson2023denovo}.
Diffusion models achieve a good balance between tractability and expressivity, where the generative process is defined as the reverse of a forward diffusion process that progressively maps the data distribution to Gaussian noise.

In standard diffusion, the noise level is the same across all dimensions. This synchronized corruption process forces \textit{all} dimensions to evolve along the same generative trajectory.
However, we can extend the forward noising process to a multi-dimensional noise vector, capturing a broader family of distributions, amortizing over a richer modeling space, and enabling control of the generative trajectory over the input.
Standard diffusion becomes the special case where all dimensions share the same time.

Recent work has shown that this synchronization can be relaxed across frames in a video~\citep{ruhe_rolling_2024,song_history-guided_2025}, different modalities~\citep{rojas_diffuse_2025}, text tokens~\citep{wu_ar-diffusion_2023,kim_dont_2025}, or other tasks~\citep{chen_diffusion_2024,gerdes_gud_2024}.
Here, we go one step further, and extend this perspective to independently noised pixels within a single image, following the intuition that different parts are uncorrelated to a certain degree.
Although each work approached this problem from a novel standpoint, the mathematical formulation behind them can be unified under the framework of \textit{joint diffusion}~\citep{rojas_diffuse_2025}, where \textit{data dimensions} can be corrupted with different noise levels.

A central practical difficulty of using independent noise levels per data dimension is timestep sampling during training.
This is because, during training, only \textit{highly heterogeneous states} are seen by the model.
Fully noisy, nearly clean, and structured clean/noisy partitions of the input are then severely underrepresented, even though they are essential for generation and inpainting.
We address this mismatch with a controlled timestep sampler used for training that preserves a target global corruption level while allowing spatially heterogeneous localised noise levels.

The model can, therefore, naturally process composite images with no noise on some parts of the image, and varying noise on the others, which is what enables \textit{inpainting} without any adhoc copy-pasting on the image being inpainted, or special finetuning and architectural changes to condition the model. Then, thanks to this flexible formulation, we find that the signal difference between a model that is conditioned on context regions and a model that is not can be amplified to increase consistency with the context regions, a process that we call \textit{input guidance}.
\begin{figure}[t]
    \centering

    \begin{subfigure}[c]{0.55\linewidth}
        \centering
        \setlength{\tabcolsep}{2pt}
        \renewcommand{\arraystretch}{1.05}
        \includegraphics[width=\linewidth]{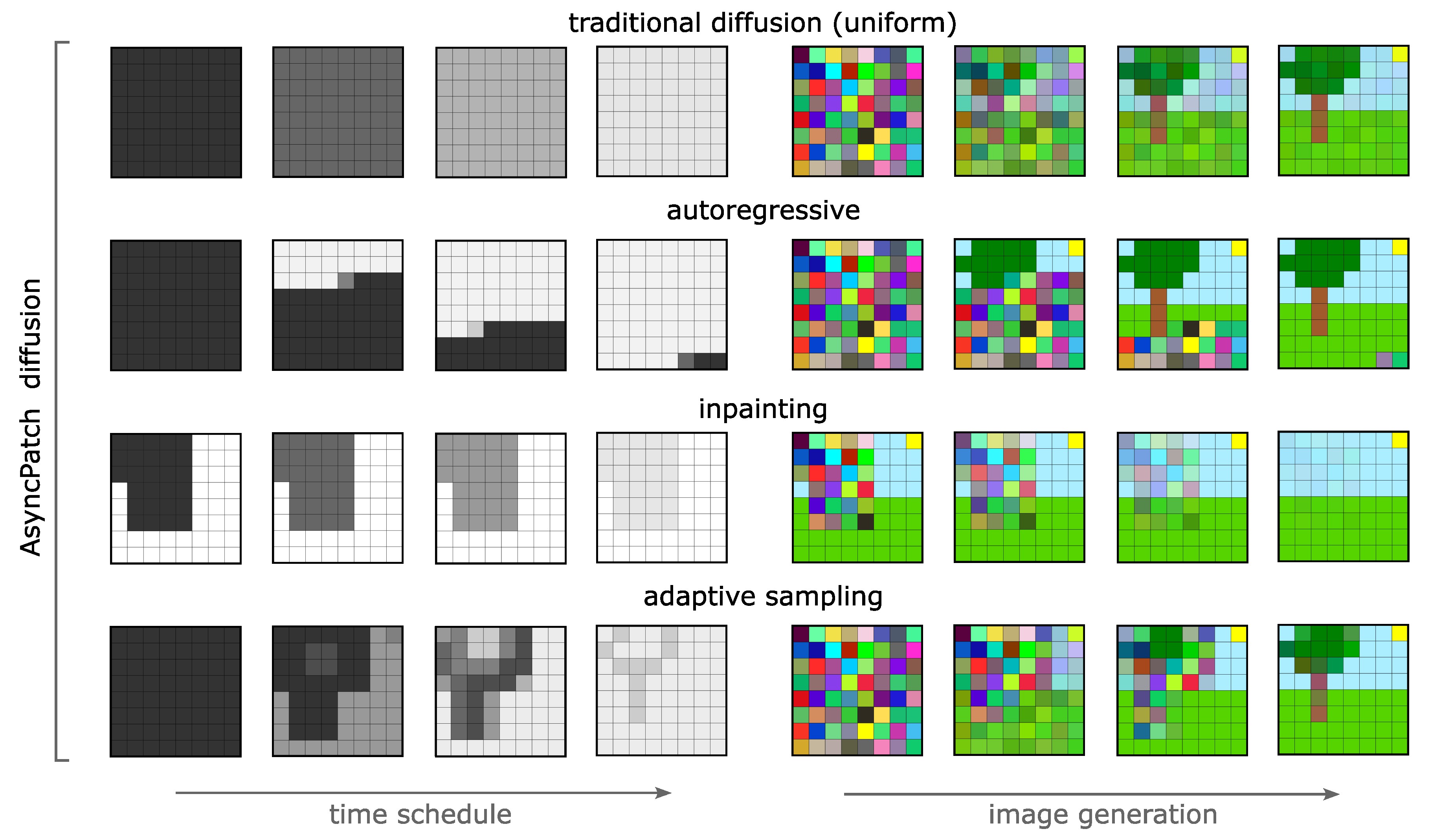}

        \caption{Comparison of sampling time schedules.}
        \label{fig:comparison-framework}
    \end{subfigure}
    \hfill
    \begin{subfigure}[c]{0.42\linewidth}
        \centering
        \includegraphics[width=\linewidth]{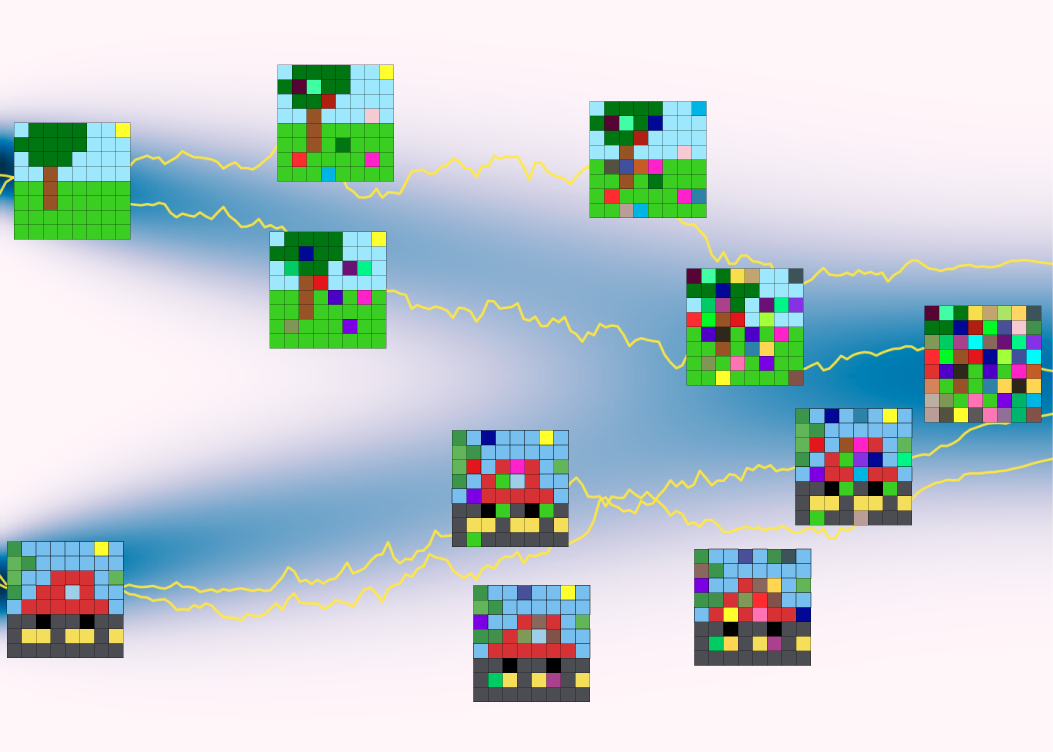}
        \caption{Example forward diffusion paths.}
        \label{fig:diffusion-paths}
    \end{subfigure}

    \caption{Comparison between different approaches to image generation with AsyncPatch, which allows full flexibility on the choice of generation path. Color indicates the status of generation, from black to white. On the right, an example of forward paths in AsyncPatch: given the same sample from the dataset, several diverse forward paths are possible.}
    \label{fig:figure_1}
    \vspace{-1em}
\end{figure}
This type of guidance becomes valuable in the context of inpainting and outpainting tasks where it helps by pushing the model to adhere more closely to the \textit{specific patterns and textures} present in the known parts of the image.

Finally, we show that this flexibility can be used to freely modify the schedule of the noise level used during sampling, down to the pixel level, enabling control over the timing at which specific areas should be generated, or even accelerate or slow down the generation of certain areas based on heuristics.
In this work, we make the following contributions:
\begin{enumerate}
    \item We rigorously extend joint diffusion to images at the pixel level, and to latent tokens, yielding a unified framework for \textit{spatially flexible image generation}, which we call \textit{AsyncPatch}. Our main theoretical contribution is to prove the first \emph{valid ELBO} for joint diffusion, which requires different approaches than classical diffusion models.
    We then show how downstream tasks such as inpainting arise as a natural conditional generation setting obtained through a particular choice of spatially varying noise level scheduling.
    \item We show that naive independent timestep sampling during training produces a distribution mismatch between the noisy states seen during training and those encountered along practical sampling paths used at evaluation. This biases the training, obtaining a model that is unlikely to be useful at inference. We correct this bias with a controlled timestep sampling strategy.
    \item We demonstrate that our training strategy preserves strong image-generation performance while enabling \emph{asynchronous spatial generation}, as validated on \emph{inpainting} and \emph{texture synthesis} benchmarks, where AsyncPatch outperforms comparable prior methods.\looseness=-1
    \item We show that this framework naturally supports adaptive and ordered generation, including uncertainty-guided sampling and autoregressive image generation.
\end{enumerate}
\section{Joint diffusion}
We tackle the general problem of generative modeling, where we aim to estimate the true distribution $p(\mathbf{x})$ of a random variable $\mathbf{x}$ given a dataset of observations $(\mathbf{x}^{i})_{i\in D}$.
In \emph{joint diffusion}, we assume that the data is split into a set of $N$ tokens $\mathbf{x} = \{\mathbf{x}_1, \mathbf{x}_2, \dots, \mathbf{x}_N\}$ such that $p(\mathbf{x}) = p(\mathbf{x}_1, \dots, \mathbf{x}_N)$ is the true underlying joint distribution~\footnote{Here, the term \textit{token} is general, and with $\mathbf{x}_k \in \mathbb{R}^{d_k}$ can have a different dimensionality $d_k$. It may refer to e.g., distinct modalities, latent variables, individual pixels from a single image, temporal frames in a sequence.}.
Then, to enable flexibility in the generative process, we adopt a decoupled time parameterization, assigning an independent time variable $t_i \in [0, 1]$ to each data token $\mathbf{x}_i$. Consequently, the forward process factorizes over the tokens given the clean data $\mathbf{x}$.
We define the marginal distribution of the noisy latent variables $\mathbf{z}_{\mathbf{t}}$ given the set of time variables $\mathbf{t} = \{t_1, \dots, t_N\} \in [0, 1]^N$ as:
\begin{equation}
q(\mathbf{z}_{\mathbf{t}} | \mathbf{x}) = \prod_{i=1}^N q(\mathbf{z}_{i, t_i} | \mathbf{x}_i, t_i) = \prod_{i=1}^N \mathcal{N}(\alpha_{t_i} \mathbf{x}_i, \sigma_{t_i}^2 \mathbf{I}),
\end{equation}
where $\mathbf x_i$ is the $i$-th token, $\mathbf{z}_{i, t_i}$ denotes the $i$-th noisy token at time $t_i$, $\alpha_{t_i}$ and $\sigma_{t_i}$ denote the signal scaling coefficient and noise standard deviation at token time $t_i$, respectively, as determined by the chosen variance-preserving diffusion noise schedule. Note that the decomposition is made possible because we chose token-independent forward processes.

The backward diffusion process will then define the generative model. We introduce a parameterized approximation $p_\theta(\mathbf{z}_{\mathbf{s}} | \mathbf{z}_{\mathbf{t}})$, obtained using a neural network $\mathbf{s}_\theta(\mathbf{z}_{\mathbf{t}}; \mathbf{t})$ trained to predict the score of the true underlying distribution from the noisy state $\mathbf{z}_{\mathbf{t}}$.
By extending the traditional diffusion objective to our multi-time setting, we train the model by minimizing the Fisher divergence:

\begin{equation}
    \mathcal{L}(\theta) = \sum_{i}\mathbb{E}_{\mathbf{x}, \mathbf{t}, \epsilon} \left[  \lambda(t_i) \left\| \nabla_{\mathbf{z}_{i, t_i}} \log q(\mathbf{z}_{i, t_i} | \mathbf{x}_i, t_i) - \mathbf{s}_{\theta}(\mathbf{z}_{\mathbf{t}}; \mathbf{t})_i \right\|^2 \right],
    \label{eq:score_loss}
\end{equation}
Notice how, while the forward process applies noise independently to each token, the neural network $\mathbf{s}_\theta$ takes as input all noisy tokens $\mathbf z_{\mathbf t}$ jointly to make its prediction.

Our first result is a direct adaptation of Theorem 1 and 2 from~\citet{rojas_diffuse_2025}. In particular, we show that the minimizers of the loss lead to a score which can be used to generate samples from the target distribution $p(\mathbf{x})$. The proof is postponed to Appendix~\ref{app:proof_lemmas}.

\begin{lemma}{Joint Optimization of Independent Tokens}{joint-opt}
\label{lemma:joint-opt}
    Under independent Gaussian forward diffusions for each token, any global minimizer $\theta^\star$ of $\mathcal{L}(\theta)$ yields $\mathbf s_{\theta^\star,i}(\mathbf z_{\mathbf t},\mathbf t) = \nabla_{\mathbf z_{i, t_i}}\log p(\mathbf z_{\mathbf t},\mathbf t), \;\; i=1,\dots,N$. Hence, any induced reverse process has terminal marginal $p(\mathbf{x})$.
\end{lemma}

With ~\autoref{lem:joint-opt} we show that, although we are now factorizing both the forward and the reverse diffusion process over the tokens, thus introducing a much more varied set of denoising paths, we can still model the desired data distribution $p(\mathbf{x})$, like in traditional diffusion models.

We now present our main theoretical contribution.
While the loss in \eqref{eq:score_loss} is minimized by the score function as shown in ~\autoref{lem:joint-opt}, we can also show that a reweighted version of \eqref{eq:score_loss} is valid Evidence Lower BOund (ELBO) for the model. In particular, we prove the following theorem which is the main result of our paper.

\begin{theorem}{Evidence Lower BOund}{elbo}
    There exist $\{\lambda_i\}_{i=1}^N$ with $\lambda_i: \ [0,1]^N \to (0,+\infty)$ for any $i \in \{1, \dots, N\}$ such that
    \begin{equation}
        \mathbb{E}_p[\log p_\theta(\mathbf{x})] \geq \sum_{i=1}^N \int_{[0,1]^N} \lambda_i(\mathbf{t}) \mathbb{E}_{\mathbf{x}, \mathbf{z}}\left[ \left\| \nabla_{\mathbf{z}_{i, t_i}} \log q(\mathbf{z}_{i, t_i} | \mathbf{x}_i, t_i) - \mathbf{s}_{\theta}(\mathbf{z}_{\mathbf{t}}; \mathbf{t})_i \right\|^2 \right] \rmd \mathbf{t} + C ,
    \end{equation}
    where $C \in \mathbb{R}$ is a constant that does not depend on $\theta$.
    In addition, $\{\lambda_i\}_{i=1}^N$ is explicit in the proof.
\end{theorem}

The proof of \autoref{thm:elbo} is in \autoref{sec:proof_of_elbo}. To the best of our knowledge \autoref{thm:elbo} is the \emph{first valid ELBO} for joint diffusion. The proof relies on concentration results of monotonic random walks in randomized environments, averaging DDPM-like ELBOs across all possible paths in $(0,1)^N$.

\section{Effective training and inference}

\subsection{Timestep sampling and training}

We have shown that the training objective is equivalent to the one from traditional diffusion. However, effective training is not achieved by simply sampling the timesteps $\mathbf{t}$ as independent uniform random variables. This is because useful noise paths end up being undersampled during training when such a naive strategy is used (see Appendix \ref{sec:appendix_timstep_sampling} for further discussion). For example, inpainting requires structured states where one contiguous region is clean, $t = 0$, while a neighboring region must be denoised from a noisy state $t>0$. Independent per-pixel sampling almost never produces such macroscopic partitions.

We explore three different methods, which are described in Figure~\ref{fig:timestep_sampling}, and displayed in Figure~\ref{fig:timestep_sampling_plots}. Each method balances the three factors differently. The first is a naive approach with different-sized patches each having their own timestep. The second method is inspired by RAD~\citep{kim_rad_2024} and uses a Perlin mask to define regions that are more similar to natural inpainting masks. Finally, we introduce AsyncPatch sampling which provides explicit control over the factors.

\paragraph{Architecture.} In all our experiments we use a UNet-based model, similar to the one proposed in~\citet{rombach2022high}, for both pixel-level and latent-level diffusion. In the architecture, FiLM modulation~\citep{perez2018film,dhariwal2021diffusion} is used to apply the conditioning to the model. Since the feature maps are already 2D, the timestep is directly fed as a 2D tensor and no further changes are required besides spatially down-sampling the timestep tensor for the inner layers (see Appendix \ref{sec:model_architecture} for more details).

\begin{figure*}[t]
\centering
\small

\begin{minipage}[t]{0.45\textwidth}
\centering

\textbf{Perlin sampling}
\vspace{0.3em}

\begin{algorithmic}[1]
\State Sample Perlin mask $\mathbf m \in \{0,1\}^{H \times W}$
\State Sample $\mathbf t \sim \mathcal{U}(\mathbf 0, \mathbf 1)$
\State Sample $b \sim \mathcal{B}(0,1)$

\State \Return $\mathbf t\mathbf m + (1-\mathbf m) b$
\end{algorithmic}

\vspace{0.8em}
\hrule
\vspace{0.8em}

\textbf{Patchwise sampling}
\vspace{0.3em}

\begin{algorithmic}[1]
\State Sample number of patches $K$
\State Partition image into patches $\{P_i\}_{i=1}^K$
\For{$i=1,\dots,K$}
    \State Sample $t_i \sim \mathcal{U}(0,1)$
    \State Set $\mathbf t|_{P_i} \gets t_i$
\EndFor
\State \Return $\mathbf t$
\end{algorithmic}

\end{minipage}
\hfill
\vrule
\hfill
\begin{minipage}[t]{0.50\textwidth}
\centering

\textbf{AsyncPatch sampling}
\vspace{0.3em}

\begin{algorithmic}[1]
\State Sample $\bar{t} \sim \mathcal{U}(t_{\min},t_{\max})$
\State Set $\delta \gets \min(\bar{t}-t_{\min},\, t_{\max}-\bar{t},\, 0.5)$
\State Set $t^- \gets \bar{t}-\delta$, \quad $t^+ \gets \bar{t}+\delta$
\State Sample number of patches $K$
\State Partition image into patches $\{P_i\}_{i=1}^K$
\For{$i=1,\dots,K$}
    \State Sample $t_i \sim \mathcal{U}(t^-,t^+)$
    \State Set $\mathbf t|_{P_i} \gets t_i$
\EndFor
\State \Return $\mathbf t$
\end{algorithmic}

\end{minipage}

\caption{
Perlin sampling produces an inpainting-like
clean/noisy partition, patchwise sampling draws independent regional
timesteps, and AsyncPatch sampling first fixes a global mean corruption level
before drawing patch-wise timesteps within the largest feasible interval. More details in Appendix \ref{sec:appendix_timstep_sampling}.
}
\label{fig:timestep_sampling}
\end{figure*}
\begin{figure}
    \centering
    \includegraphics[width=0.3\linewidth]{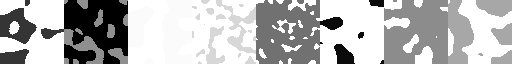}\hfill
    \includegraphics[width=0.3\linewidth]{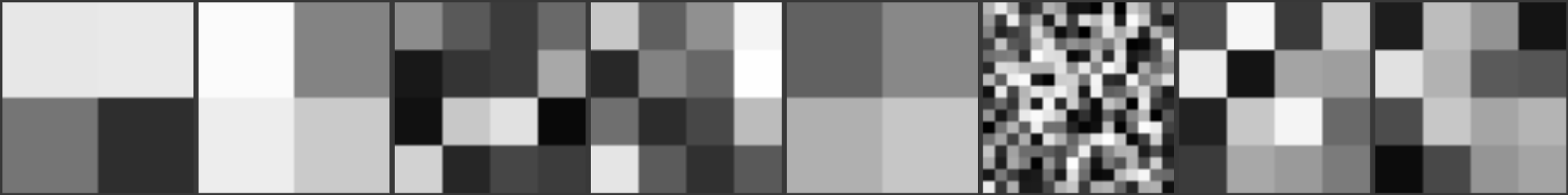}\hfill
    \includegraphics[width=0.3\linewidth]{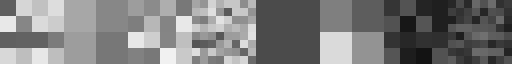}

    \includegraphics[width=0.3\linewidth]{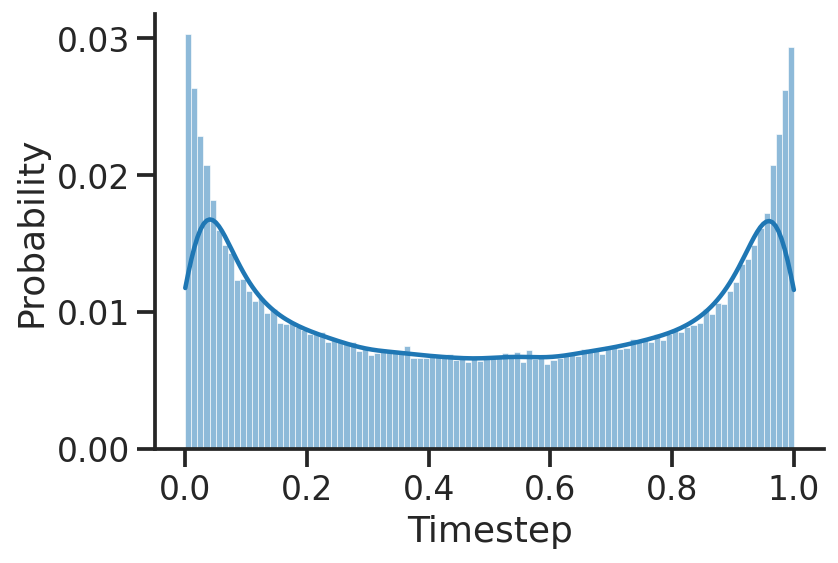}\hfill
    \includegraphics[width=0.3\linewidth]{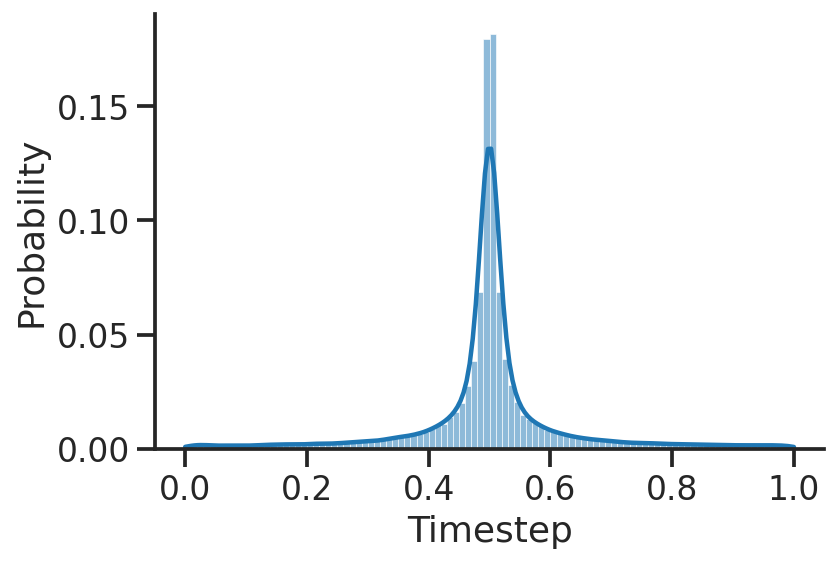}\hfill
    \includegraphics[width=0.3\linewidth]{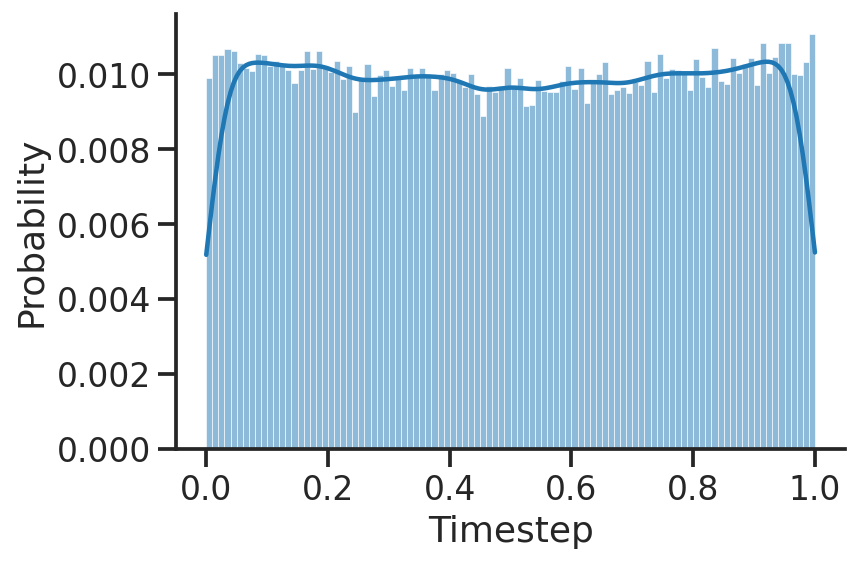}
    \caption{Sampled training timesteps and distribution of the mean timestep per image. From left to right: Perlin sampling, patchwise sampling, and AsyncPatch sampling.}
    \label{fig:timestep_sampling_plots}
\end{figure}

\subsection{Controlled inference-time sampling}
\label{subsec:controlled_inference_time_sampling}

We can now define different noise schedules during sampling to choose a specific reverse-time path and obtain the desired behavior. Let $\boldsymbol{\tau}(u) = \{\tau_1(u), \dots , \tau_N(u)\} \in [0,1]^N, \; u \in [0,1]$ be a spatial noise schedule, where each token will now be generated according to the monotonically decreasing schedule $\tau_i(u)$. Note that this schedule need not start from $1$, instead, we can set some tokens to $0$ if we already assume to know them, or choose any $\tau_i(0) \in [0,1]$ as a starting point. See Fig.~\ref{fig:figure_1}.

\paragraph{Traditional diffusion.} This is obtained by simply setting $\tau_1(u) = \dots =\tau_N(u)$.

\paragraph{Autoregressive generation.} Partition the tokens into ordered groups $G_1,\dots,G_K$. During stage $k$, only the tokens in $G_k$ are denoised, while previously generated tokens are kept clean and future tokens remain fully noisy. Concretely, for $i\in G_k$ we may define
\begin{equation*}
\tau_i(u)=
    \begin{cases}
        1, & u < \frac{k-1}{K},\\
        K\left(\frac{k}{K}-u\right), & \frac{k-1}{K}\leq u \leq \frac{k}{K},\\
        0, & u > \frac{k}{K}.
    \end{cases}
\end{equation*}

Thus the sampler first generates $G_1$, then $G_2$, and so on. At each stage, the model updates the active group conditioned on the already denoised groups and on the still-noisy unresolved groups.

\paragraph{Inpainting.} Let $\mathbf x=\{\mathbf x_i\}_{i=1}^N$ be the masked image and let $\mathbf m\in\{0,1\}^N$ be a binary mask, where $m_i=1$ denotes an observed token and $m_i=0$ denotes a missing token. Then, the initial state is obtained by keeping the observed tokens clean and initializing the missing tokens from noise $\mathbf z_{i,\tau_i(0)} = m_i \mathbf x_i + (1 - m_i) \boldsymbol{\epsilon}_i$, with $\epsilon_i \sim \mathcal N(\mathbf 0,\mathbf I)$.
The corresponding inpainting schedule is
\begin{equation}
    \tau_i(u) = m_i \cdot 0 + (1 - m_i) (1-u), \qquad u\in[0,1].
\end{equation}
Thus, observed tokens remain at $\tau_i=0$, while missing tokens are denoised from $\tau_i=1$ to $\tau_i=0$. At reverse time $u$, the noisy image can be written tokenwise as $\mathbf z_{i,\tau_i(u)}=m_i \mathbf x_i+(1 - m_i) \mathbf z_{i,1-u}$.
During sampling, the model predicts the joint score over all tokens, but the reverse update is applied only to the missing region: $\frac{d \mathbf z_{i,t_i(u)}}{du} \propto (1 - m_i)\, \mathbf s_{\theta} \left( \mathbf z_{\boldsymbol{\tau}(u)}; \boldsymbol{\tau}(u)\right)_i$.
Because the observed tokens are fixed to the clean values $y_i$, and thanks to constant factors being removed by the gradient, the masked update corresponds to the conditional score of the missing region given the observed region:
\begin{equation}
    \nabla_{\mathbf z_{i,\tau_i}}
    \log p_{\mathbf t}
    \left(
        \mathbf z_{\boldsymbol{\tau}(u)}
    \right)
    =
    \nabla_{\mathbf z_{i,\tau_i}}
    \log p_{\mathbf t}
    \left(
         (1 - \mathbf m) \, \mathbf z_{\boldsymbol{\tau}(u)}
        \mid
        \mathbf m \mathbf x
    \right).
\end{equation}
Therefore, inpainting is obtained by choosing a spatial noise schedule that keeps the known region clean and denoises only the complementary mask, while no task-specific inpainting objective is required.\looseness=-1

\subsection{Input guidance}
Input guidance is a direct generalisation of noisy guidance, as proposed in~\citet{rojas_diffuse_2025}. It allows the guidance signal to be chosen from the input itself by comparing score estimates at two different noise levels. Let $\mathbf{s}$ denote the smaller, cleaner timestep and let $\mathbf{t}$ denote the larger, noisier timestep, with $\mathbf{s} \leq \mathbf{t}$. Given the corresponding noisy inputs $\mathbf{z}_{\mathbf{s}}$ and $\mathbf{z}_{\mathbf{t}}$, input guidance modifies the score by pushing the prediction at the cleaner timestep away from the prediction obtained from the more corrupted input:
$\bar{\mathbf{s}}_\theta(\mathbf{z}_{\mathbf{s}}; \mathbf{s}) = (1+\omega_i)\mathbf{s}_\theta(\mathbf{z}_{\mathbf{s}}; \mathbf{s}) - \omega_i \mathbf{s}_\theta(\mathbf{z}_{\mathbf{t}}; \mathbf{t})$, where $\omega_i \geq 0$ controls the strength of input guidance.

In practice, we use a class-conditional diffusion model with score
$\mathbf{s}_\theta(\mathbf{z}_{\mathbf{s}}, c; \mathbf{s})$, where $c$
denotes the class label. We write $\emptyset$ for the null condition used by
classifier-free guidance. Classifier-free guidance is first applied
independently at both noise levels:
\begin{align}
    \tilde{\mathbf{s}}_\theta(\mathbf{z}_{\mathbf{s}}, c; \mathbf{s})
    &=
    (1+\omega_c)\mathbf{s}_\theta(\mathbf{z}_{\mathbf{s}}, c; \mathbf{s})
    -
    \omega_c \mathbf{s}_\theta(\mathbf{z}_{\mathbf{s}}, \emptyset; \mathbf{s}), \\
    \tilde{\mathbf{s}}_\theta(\mathbf{z}_{\mathbf{t}}, c; \mathbf{t})
    &=
    (1+\omega_c)\mathbf{s}_\theta(\mathbf{z}_{\mathbf{t}}, c; \mathbf{t})
    -
    \omega_c \mathbf{s}_\theta(\mathbf{z}_{\mathbf{t}}, \emptyset; \mathbf{t}),
\end{align}
where $\omega_c \geq 0$ is the classifier-free guidance weight. Input guidance is then applied to the class-guided scores:
\begin{equation}
    \bar{\mathbf{s}}_\theta(\mathbf{z}_{\mathbf{s}}, c; \mathbf{s})
    =
    (1+\omega_i)\tilde{\mathbf{s}}_\theta(\mathbf{z}_{\mathbf{s}}, c; \mathbf{s})
    -
    \omega_i \tilde{\mathbf{s}}_\theta(\mathbf{z}_{\mathbf{t}}, c; \mathbf{t}).
\end{equation}
Class guidance controls the strength of the conditional signal, while input guidance controls how strongly the score at the cleaner timestep is repelled from the score predicted using the noisier input.

\section{Related Work}
\label{gen_inst_related_work}

Most closely related to AsyncPatch, is the concurrent work on Patch Forcing~\cite{schusterbauer2026patchforcing}.
In this work the authors also propose patch-level denoising schedules and address the same train-test mismatch created by independently sampled spatial timesteps.
Their main emphasis however is on the efficiency gains from adaptive sampling, and their proposed solutions, model class, theoretical contribution, and experimental emphasis depart from our analysis.
See Appendix \ref{app_related_work} for a more in-depth discussion and extended related work.

\paragraph{Diffusion with heterogeneous timesteps.}
Standard diffusion models use a single global noise level shared across all dimensions or tokens \citep{sohl2015deep,ho2020denoising}. Recent work shows that this synchronization can be relaxed by assigning separate noise levels to different components, including frames in video generation \citep{ruhe_rolling_2024,song_history-guided_2025}, modalities \citep{rojas_diffuse_2025}, text tokens \citep{wu_ar-diffusion_2023,kim_dont_2025}, and sequential prediction tasks \citep{chen_diffusion_2024,gerdes_gud_2024}. Multimodal diffusion models such as \textit{UniDiffuser} \citep{bao_unidiffuser_2023}, \textit{AVDiT} \citep{kim_avdit_2024}, \textit{OmniFlow} \citep{li_omniflow_2025}, and \textit{UniDisc} \citep{swerdlow2025unidisc} further exploit modality-specific corruption schedules for flexible any-to-any generation. Related ideas also appear in protein sequence-structure co-generation, including \textit{MultiFlow} \citep{campbell_multiflow_2024} and \textit{Generator Matching} \citep{holderrieth_generator_2025}, as well as self-supervised flow matching with heterogeneous token corruption \citep{chefer_selfflow_2026}. In the spatial domain,~\citet{wewer_spatial_2025} adapt diffusion forcing to spatial variables for reasoning tasks. Our work instead focuses on spatially heterogeneous noise level fields within a single image, targeting generic asynchronous generation while preserving full-image generative capabilities.

\paragraph{Diffusion-based inpainting and spatial masking.}
Diffusion models have become a standard approach for image inpainting and masked generation \citep{lugmayr_repaint_2022,rombach2022high}. Recent methods extend diffusion editing with stronger spatial conditioning and mask-aware generation, including \textit{GradPaint} \citep{grechka2024gradpaint}. Closest to our setting, \textit{RAD} \citep{kim_rad_2024} introduces pixel-dependent timesteps generated from Perlin masks for diffusion inpainting. However, these approaches primarily target reconstruction or editing under fixed masked regions. In contrast, we train diffusion models under general spatially heterogeneous timestep schedules, enabling both localized reconstruction and unconstrained image generation.

\paragraph{Image Inpainting}
Image inpainting has evolved from supervised architectures specialized for masked completion, such as \textit{LaMa} \citep{suvorov_resolution-robust_2021}, to diffusion-based methods that better capture the multimodal uncertainty of large missing regions. Diffusion models can be fine-tuned for inpainting \citep{rombach2022high, saharia2022palette, xie2025turbofill, manukyan2023hd} based on a dataset of masks, but they are also often adapted to in-painting via \emph{zero-shot} methods at sampling time, such as \textit{GLIDE}\citep{nichol_glide_2022}, \textit{RePaint} \citep{lugmayr_repaint_2022}, \textit{DiffEdit}\citep{couairon_diffedit_2022}, \textit{FLUX}\citep{labs_flux1_2025}, Diptych Prompting \citep{shin2025large}, Decoupled diffusion guidance \citep{moufad2025efficient}. Our work supports inpainting natively, with spatial flexibility directly integrated into its diffusion parameterization.

\section{Experiments}

\subsection{Timestep sampling during training}

In this section, we investigate the importance of the timestep sampling strategy on the model's performance in both image generation and inpainting. We perform these experiments using a pixel-level 60M parameter model on the ImageNet~\citep{DenDon09Imagenet} 64 dataset, which is obtained by resizing ImageNet to $64\times64$ pixels after center-cropping the image to a square. From Tab.~\ref{tab:ablation_study} (b), we observe how using a timestep sampling strategy that matches more closely the inpainting task results in good inpainting performance, but suffers in the traditional generation task. The patchwise approach is equally good at inpainting, but generation suffers. Instead, with AsyncPatch, we strike a good balance of generation performance, which gets close to the baseline, and inpainting. Results from Tab.~\ref{tab:ablation_study} (a) use a full 600M parameter model for better comparison. See Appendix~\ref{sec:samples_imagenet_64} for samples.

\begin{table}[h!]
    \caption{Ablation study on ImageNet 64 comparing generation (FID) and inpainting (LPIPS) performance.}
    \label{tab:ablation_study}
    \centering

    \begin{subtable}[t]{0.178\textwidth}
        \centering
        \caption{Sampling FID on 50k samples}
        \label{tab:fid_scores}

        \resizebox{\columnwidth}{!}{
        \begin{tabular}{lr}
            \toprule
            \multicolumn{2}{c}{Generation}  \\
            Method & FID \\
            \midrule
            Baseline & \textbf{1.74} \\
            \midrule
            Perlin & 2.05 \\
            Patchwise & 2.54 \\
            AsyncPatch & \underline{1.77} \\
            \bottomrule
        \end{tabular}
        }
    \end{subtable}
    \hfill
    \begin{subtable}[t]{0.8\textwidth}
        \centering
        \caption{Combined quantitative comparison on ImageNet64 inpainting evaluation.
        FID and LPIPS are computed on 1k samples. Lower is better for both metrics.}
        \label{tab:combined_results_imagenet64_inpainting}
        \resizebox{\columnwidth}{!}{
        \begin{tabular}{lrrrrrrrrr}
        \toprule
         & \multicolumn{2}{c}{Extrema}
         & \multicolumn{2}{c}{Square}
         & \multicolumn{2}{c}{Thin}
         & \multicolumn{2}{c}{Wide} \\
         Method
         & FID $\downarrow$ & LPIPS $\downarrow$
         & FID $\downarrow$ & LPIPS $\downarrow$
         & FID $\downarrow$ & LPIPS $\downarrow$
         & FID $\downarrow$ & LPIPS $\downarrow$ \\
        \midrule
        Baseline (RePaint)
         & 56.4 & 0.282
         & 31.8 & 0.106
         & 23.2 & 0.053
         & 29.1 & 0.097 \\
        \midrule
        Perlin
         & \textbf{49.6} & \textbf{0.225}
         & \textbf{28.2} & \textbf{0.085}
         & \textbf{19.7} & \textbf{0.039}
         & \textbf{24.6} & \textbf{0.075} \\

        Patchwise
         & \underline{49.7} & \underline{0.228}
         & \underline{28.7} & \underline{0.087}
         & \underline{20.4} & \underline{0.041}
         & \underline{25.6} & \underline{0.078} \\

        AsyncPatch
         & 50.0 & 0.229
         & 29.1 & 0.088
         & \underline{20.4} & \underline{0.041}
         & 26.0 & \underline{0.078} \\
        \bottomrule
        \end{tabular}
        }
    \end{subtable}
\end{table}

\begin{wraptable}[5]{r}{0.4\textwidth}
\vspace{-3.5em}
    \centering
    \caption{FID on 50k samples from validation set of ImageNet 256 and LSUN bedroom with no CFG.}
    \label{tab:fid_scores_imagenet256}
    \resizebox{0.4\columnwidth}{!}{
    \begin{tabular}{lrrr}
        \toprule
        Method & ImageNet & LSUN Bedroom \\
        \midrule
        LDM &  8.24  & 2.83 \\
        AsyncPatch & 8.06 & 3.09  \\
        \bottomrule
    \end{tabular}
    }
\end{wraptable}
\subsection{Latent Diffusion Model}
We train a latent diffusion model (LDM) version of AsyncPatch, and a baseline LDM on both the ImageNet 256~\citep{DenDon09Imagenet} and LSUN Bedroom~\citep{YuZSSX15LSUN} datasets for the same number of steps and using the same batch size (see Appendix~\ref{sec:model_architecture}). When inpainting is performance, the encoder takes masked images to produce latents. For AsyncPatch to work in this setting we use an autoencoder that is localized, such that no information regarding the mask is leaked to the surrounding latents. To achieve this, we train an autoencoder on masked images and explicitly condition it on the mask. From our experiments, we found that this technique is sufficient to stop the encoder from storing information regarding the mask, and allows AsyncPatch to be applied for inpainting. The autoencoder downsamples the images from $256\times 256 \times 3$ down to a latent space of $64\times 64\times 3$ (see Appendix \ref{sec:model_architecture} for details). We observed that, using AsyncPatch timestep sampling during training, generation performance does not suffer, and is even better in some cases (see Tab.~\ref{tab:fid_scores_imagenet256}).

\subsection{Inpainting}

\begin{table*}[bthp]
\begin{minipage}[t]{0.49\textwidth}
\centering
\caption{Combined quantitative comparison on LSUN bedroom\textsuperscript{\dag}.
All other results are from~\citet{kim_rad_2024}, and are reported in pixel space. FID and LPIPS are computed on 1k samples.}
\label{tab:combined_results_LSUN}
\resizebox{\linewidth}{!}{
\begin{tabular}{lrrrrrrrr}
\toprule
 & \multicolumn{2}{c}{Extrema} & \multicolumn{2}{c}{Square} & \multicolumn{2}{c}{Thin} & \multicolumn{2}{c}{Wide} \\
 Method & FID $\downarrow$ & LPIPS $\downarrow$ & FID $\downarrow$ & LPIPS $\downarrow$ & FID $\downarrow$ & LPIPS $\downarrow$ & FID $\downarrow$ & LPIPS $\downarrow$ \\
\midrule
Supervised & 20.1 & 0.397 & 13.3 & 0.129 & 8.2 & 0.042 & 10.7 & 0.096 \\
\midrule
Score-SDE & 24.1 & 0.648 & 23.7 & 0.648 & - & - & 23.2 & 0.644 \\
DDRM      & 33.1 & 0.450 & 20.5 & 0.166 & - & - & 26.4 & 0.190 \\
MCG       & 22.0 & \textbf{0.395} & 19.9 & \textbf{0.131} & - & - & 20.9 & 0.108 \\
DDNM      & 53.3 & 0.431 & 22.7 & 0.150 & - & - & 23.2 & 0.126 \\
DeqIR     & 43.9 & 0.461 & 22.2 & 0.176 & - & - & 22.0 & 0.153 \\
RAD       & 21.6 & 0.399 & 19.2 & 0.131 & - & - & 20.8 & \textbf{0.107} \\

RePaint & 23.5 & 0.461 & 20.5 & 0.176 & - & - & 21.4 & 0.161 \\
\midrule
RePaint\textdagger & 23.7 & 0.499 & 15.6 & 0.171 & 13.9 & 0.121 & 14.5 & 0.164 \\
AsyncPatch\textdagger & \textbf{20.2} & 0.406 & \textbf{14.5} & 0.145 & \textbf{13.3} & \textbf{0.107} & \textbf{13.9} & 0.141 \\
\bottomrule
\end{tabular}
}
\end{minipage}
\hfill
\begin{minipage}[t]{0.49\textwidth}
\centering
\caption{Combined quantitative comparison on ImageNet 256\textsuperscript{\dag}.
All other results are from~\citet{kim_rad_2024}, and are reported in pixel space. FID and LPIPS are computed on 1k samples.}
\label{tab:combined_results_imagenet}
\resizebox{\linewidth}{!}{
\begin{tabular}{lrrrrrrrr}
\toprule
 & \multicolumn{2}{c}{Extrema} & \multicolumn{2}{c}{Square} & \multicolumn{2}{c}{Thin} & \multicolumn{2}{c}{Wide} \\
 Method & FID $\downarrow$ & LPIPS $\downarrow$ & FID $\downarrow$ & LPIPS $\downarrow$ & FID $\downarrow$ & LPIPS $\downarrow$ & FID $\downarrow$ & LPIPS $\downarrow$ \\
\midrule
Supervised & 38.7 & 0.406 & 26.1 & 0.155 & 14.6 & 0.077 & 16.4 & 0.112 \\
\midrule
Score-SDE & 86.6 & 0.495 & 57.2 & 0.200 & - & - & 62.0 & 0.183 \\
DDRM      & 106.9 & 0.492 & 74.3 & 0.224 & - & - & 75.1 & 0.211 \\
MCG       & 58.4 & 0.448 & 48.1 & 0.132 & - & - & 56.9 & 0.124 \\
DDNM      & 80.6 & 0.476 & 63.8 & 0.187 & - & - & 64.5 & 0.167 \\
DeqIR     & 99.5 & 0.505 & 66.5 & 0.195 & - & - & 68.4 & 0.182 \\
RAD       & 57.8 & \textbf{0.374} & 47.0 & \textbf{0.118} & - & - & 56.7 & \textbf{0.104} \\

RePaint & 84.0 & 0.479 & 54.0 & 0.177 & - & - & 59.0 & 0.166 \\
\midrule
RePaint\textdagger & 45.7 & 0.504 & 32.2 & 0.206 & 31.1 & 0.202 & 27.6 & 0.198 \\
AsyncPatch\textdagger & \textbf{39.0} & 0.413 & \textbf{26.6} & 0.162 & \textbf{21.2} & \textbf{0.133} & \textbf{22.1} & 0.152 \\
\bottomrule
\end{tabular}
}
\end{minipage}
\vspace{0.5em}
\begin{minipage}{0.98\textwidth}
\footnotesize
\textsuperscript{\dag}Latent-diffusion-based methods, for which reported FID and LPIPS are affected by the autoencoding stage and are therefore not directly comparable to pixel-space methods.
\end{minipage}
\end{table*}

\begin{figure}[htbp]
    \centering

    \makebox[0.48\linewidth]{
        \makebox[0.12\linewidth][c]{\scriptsize Masked}
        \makebox[0.12\linewidth][c]{\scriptsize Supervised}
        \makebox[0.12\linewidth][c]{\scriptsize RePaint}
        \makebox[0.12\linewidth][c]{\scriptsize AsyncPatch}
    }%
    \hfill%
    \makebox[0.48\linewidth]{
        \makebox[0.12\linewidth][c]{\scriptsize Masked}
        \makebox[0.12\linewidth][c]{\scriptsize Supervised}
        \makebox[0.12\linewidth][c]{\scriptsize RePaint}
        \makebox[0.12\linewidth][c]{\scriptsize AsyncPatch}
    } \\
    \includegraphics[width=0.48\linewidth]{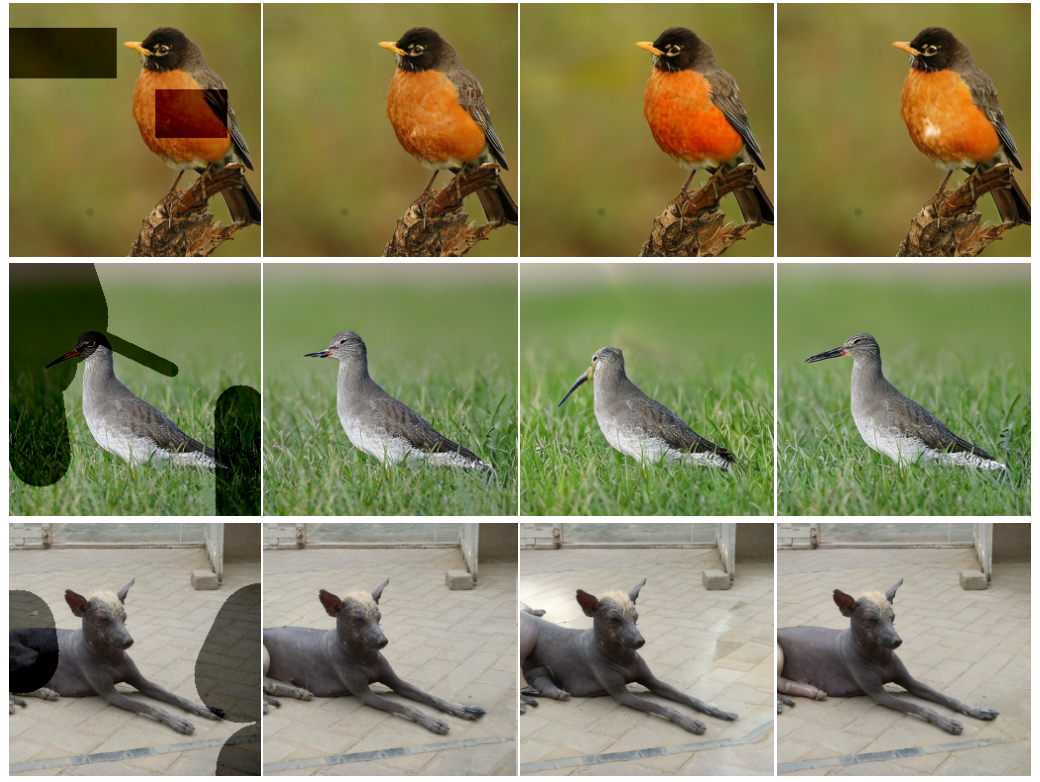}\hfill
    \includegraphics[width=0.48\linewidth]{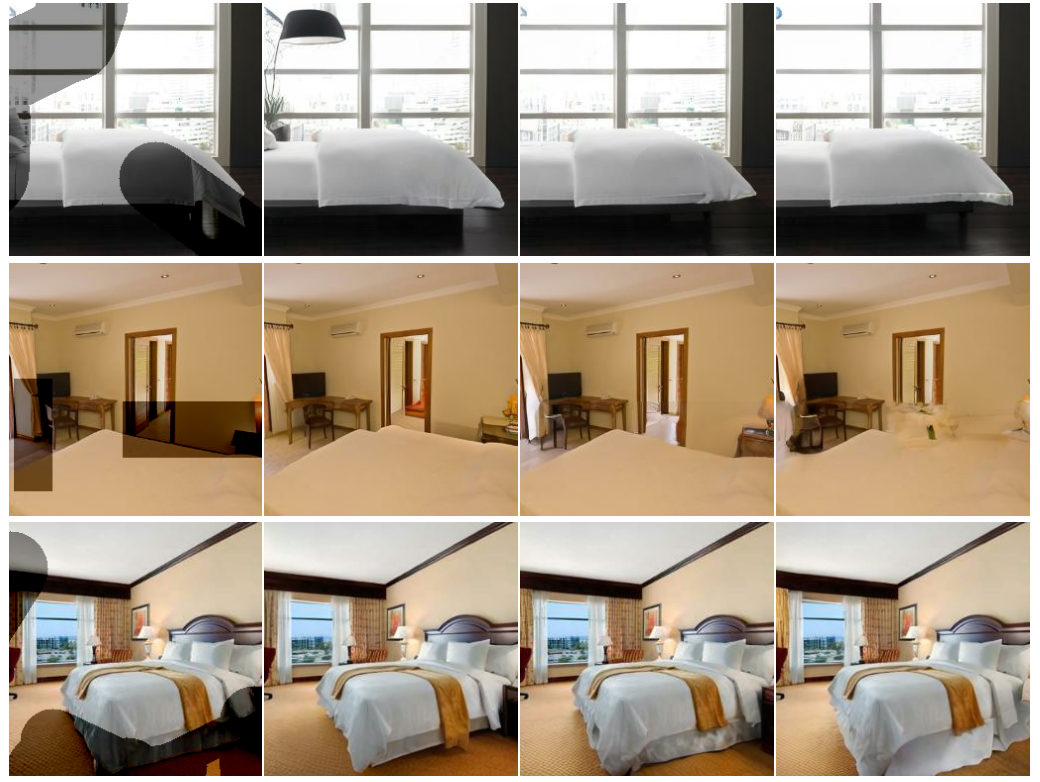}
    \caption{Qualitative performance of the models on ImageNet 256 and LSUN bedroom. Shown are the models that operate on latent space with Wide masks.}
    \label{fig:inpainting_comparison}
\end{figure}
We perform inpainting on ImageNet 256 and LSUN bedroom datasets. We compare to existing baselines and use the same evaluation setup as in~\citet{kim_rad_2024}. Importantly, the previous methods are applied to pixels-space inpainting, while we perform latent-space inpainting. From our experiments the LPIPS metric is very sensitive to the precise pixel values, which are affected by the compression performed by the autoencoder. As a reference, the images re-encoded using the autoencoder from~\citet{rombach2022high} have an LPIPS of $0.07$ on ImageNet 256. Additionally, this effect can be seen by the FID score being low for the LDM version of RePaint, while the pixel version has better LPIPS.
AsyncPatch, instead, uses the pre-trained model with no additional fine-tuning for inpainting. It maintains generative performance, and improves on previous baselines for inpainting quality. This is thanks to the joint diffusion framework, and the careful timestep sampling method used during training.

The supervised baseline was trained on square, thin, and wide masks, with extrema being out of distribution. From the results, we can see how, for in distribution masks, the supervised baseline significantly outperforms AsyncPatch. Importantly, for the out of distribution mask, AsyncPatch matches the performance of the supervised baseline.
AsyncPatch is successfully solving the inpainting task, demonstrating spatially-flexible image generation.

\subsection{Input guidance and texture synthesis}
\begin{figure}[t]
\centering

\includegraphics[width=\linewidth]{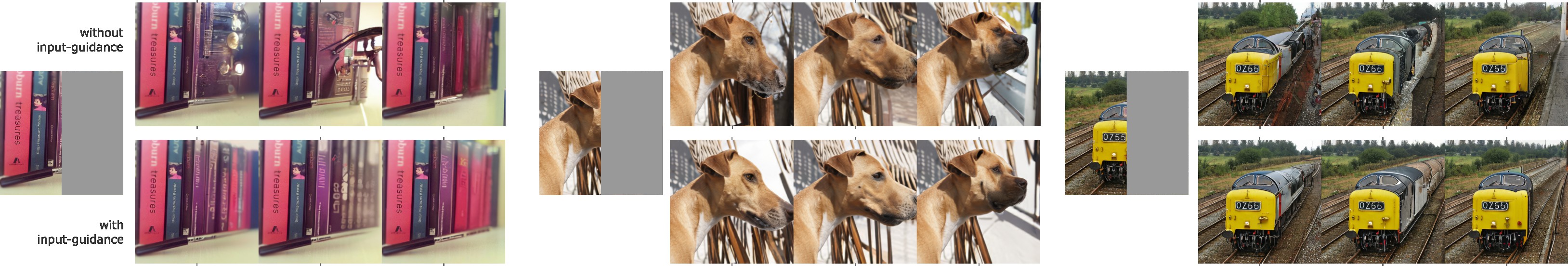}
\caption{Effect of input guidance (0.0 for top vs 2.0 for bottom) in the inpainting of the right part of the images. Three seeds are shown for each example. Guidance leads to more coherent inpainting that match the known parts.}
\label{fig:input-guidance-merged}
\vspace{-1em}
\end{figure}

We experiment with input guidance by showing its effect using different seeds on the same image and increasing strength of guidance. In Fig.~\ref{fig:input-guidance-merged} we  notice how input guidance ensures more consistent inpainting of the unknown region, with details that more closely follow the known parts.

\begin{SCfigure}[][hbtp]

    \vspace{2em}
    \adjincludegraphics[width=0.6\linewidth]{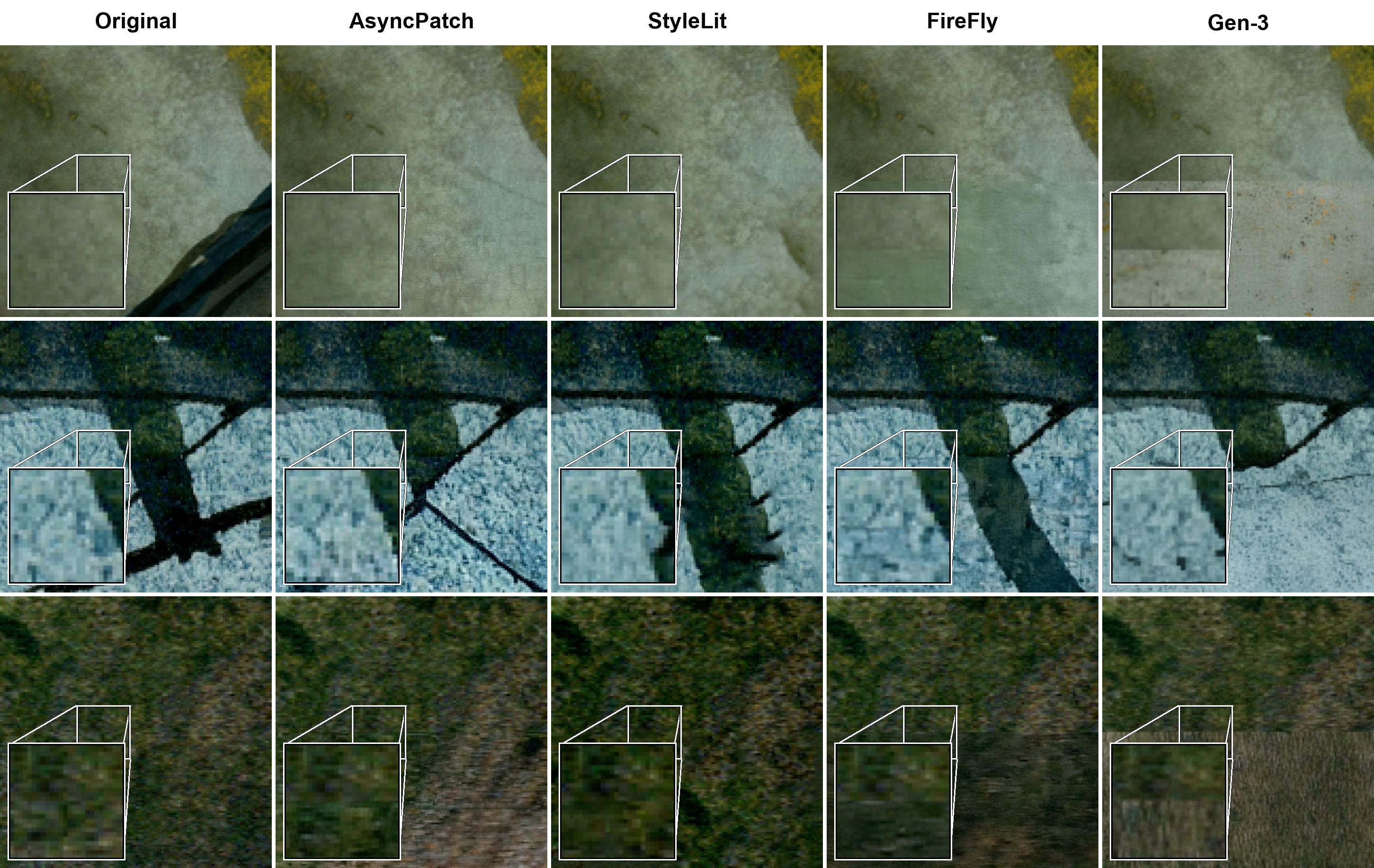}
    \caption{Qualitative comparison of texture synthesis. The top half of the texture was given, while the bottom was synthesized using our input guidance approach and other patch-based and generative methods (StyLit algorithm -- state-of-the-art patch-based method~\citet{fiser_stylit_2016} -- and diffusion-based generative models: Adobe's Firefly 3 and Runway’s Gen-3). Note how input guidance better reproduces the original texture pattern and maintains a higher-level structural consistency in contrast to other diffusion-based methods. Third row brighter for visualization. See Appendix~\ref{sec:app_texture_synthesis} for more examples.}
    \label{fig:stylization}
\end{SCfigure}

\paragraph{Texture Synthesis.} Another notable effect of input guidance is its ability to more closely follow the particular pattern and textures present in the known part of the image. This enables the model to break away from textures and patterns seen during training, and instead focus on those given. In the literature, this is known as a texture synthesis problem~\cite{efros_texture_1999,kwatra_texture_2005,kaspar_self_2015}, of which popular applications include, e.g., content-aware fill~\cite{barnes_patchmatch_2009} or example-based stylization~\cite{fiser_stylit_2016}. Traditional texture synthesis techniques rely on a patch-based principle~\cite{efros_image_2001} whose aim is to build a mosaic by seamlessly stitching irregularly shaped patches taken from a given exemplar. Consequently, since the output consists of piecewise copies of the original texture, its pixel-level appearance is reproduced perfectly (intuitively ``nothing is better than an exact copy''). Generative techniques can also be applied in this context (see e.g., Adobe Photoshop's Generative Fill, powered by the Firefly 3 model, or Runway’s image outpainting, based on the Gen-3 model). They are typically more robust at preserving high-level structures, however, they often struggle to reproduce the crucial pixel-level visual characteristics of a given texture exemplar. This leads to visual inconsistencies that are particularly visible namely when compared to patch-based methods that make use of exact local copies (see the comparison in Figure~\ref{fig:stylization}). However, by employing our input guidance mechanism, we can push the model to follow the provided data more accurately and consistently, while still being able to better preserve high-level structural consistency than traditional patch-based methods (see ~Figure~\ref{fig:stylization}).

\subsection{Flexible sampling experiments}
To further take advantage of the spatially-flexible image generation, we experiment with different sampling strategies. First, we experiment with a simple autoregressive generation. We achieve this by performing a raster scan, as described in Sec.~\ref{subsec:controlled_inference_time_sampling}. Here, we use $16$ steps per $8\times8$ patch, resulting in $1024$ total steps of sampling. Although generation quality suffers, no artifacts are present and subjects are clearly visible.

Then, we experiment with adding an uncertainty estimation head on top of a pre-trained ImageNet 64 AsyncPatch model. We train it to predict the error that the model is making, providing a good proxy for generation difficulty. Then, we sample images by accelerating sampling where uncertainty is low. Intuitively, certain parts of the input have less detail, which should require fewer steps to generate.
Here we have experimented with a simple heuristic, with no further fine-tuning of the main model.
\providecommand{\autoregressivedir}{paper_figures/autoregressive}

\begin{figure}[htbp]
    \centering
    \begin{minipage}[b]{0.32\textwidth}

        \begin{subfigure}{\linewidth}
            \centering
            {
                \setlength{\tabcolsep}{1pt}
                \newcommand{\arprogressimage}[2]{%
                    \includegraphics[width=0.19\linewidth]{\autoregressivedir/id_#1/#2.png}%
                }
                \begin{tabular}{@{}ccccc@{}}
                    {\scriptsize $t_1$} &
                    {\scriptsize $t_2$} &
                    {\scriptsize $t_3$} &
                    {\scriptsize $t_4$} &
                    {\scriptsize Final} \\
                    \arprogressimage{01}{interm_00} &
                    \arprogressimage{01}{interm_05} &
                    \arprogressimage{01}{interm_10} &
                    \arprogressimage{01}{interm_15} &
                    \arprogressimage{01}{decoded} \\
                    \arprogressimage{02}{interm_00} &
                    \arprogressimage{02}{interm_05} &
                    \arprogressimage{02}{interm_10} &
                    \arprogressimage{02}{interm_15} &
                    \arprogressimage{02}{decoded} \\
                    \arprogressimage{06}{interm_00} &
                    \arprogressimage{06}{interm_05} &
                    \arprogressimage{06}{interm_10} &
                    \arprogressimage{06}{interm_15} &
                    \arprogressimage{06}{decoded} \\
                    \arprogressimage{03}{interm_00} &
                    \arprogressimage{03}{interm_05} &
                    \arprogressimage{03}{interm_10} &
                    \arprogressimage{03}{interm_15} &
                    \arprogressimage{03}{decoded}
                \end{tabular}
            }
            \caption{Autoregressive sampling.}
            \label{fig:autoregressive_progression}
        \end{subfigure}
    \end{minipage}
    \hfill
    \begin{minipage}[b]{0.32\textwidth}
        \centering
        \begin{subfigure}{\linewidth}
            \centering
            \includegraphics[width=0.8\linewidth]{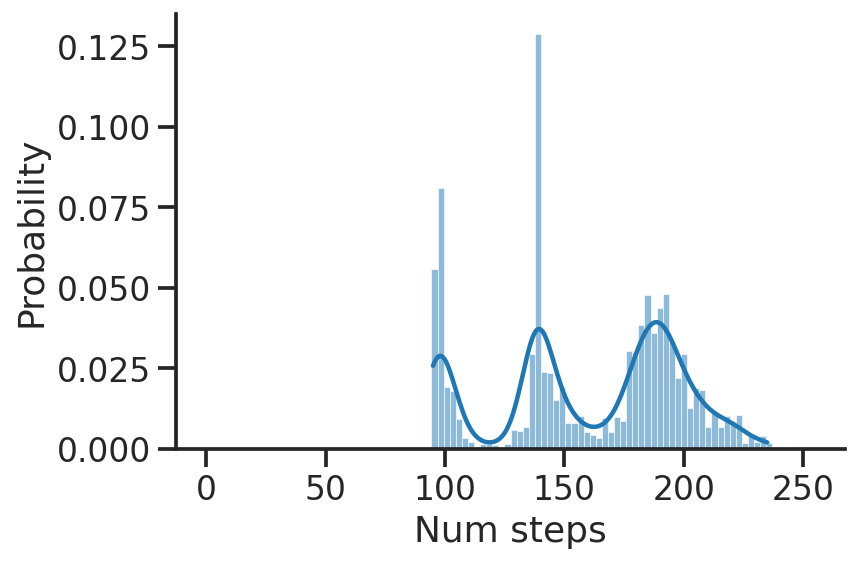}
            \caption{Number of steps}
            \label{subfig:steps}
        \end{subfigure}
        \vspace{0.2cm}
        \begin{subfigure}{\linewidth}
            \centering
            \includegraphics[width=0.8\linewidth,trim={0 4.52cm 0 0}, clip]{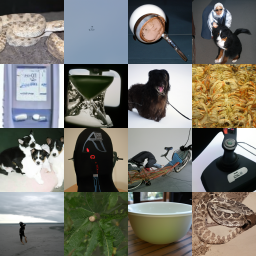}
            \caption{Samples}
            \label{subfig:samples}
        \end{subfigure}
    \end{minipage}
    \hfill
    \begin{minipage}[b]{0.32\textwidth}
        \centering
        \begin{subfigure}{\linewidth}
            \centering
            \includegraphics[width=\linewidth]{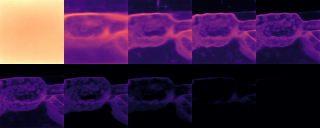}
            \caption{Uncertainty}
            \label{subfig:uncertainty}
        \end{subfigure}
        \vspace{0.2cm}
        \begin{subfigure}{\linewidth}
            \centering
            \includegraphics[width=\linewidth]{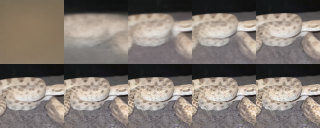}
            \caption{$x_0$ prediction}
            \label{subfig:x0_prediction}
        \end{subfigure}
    \end{minipage}
    \caption{On the left: autoregressive sampling from full noise to sample. In the middle and right: uncertainty-based accelerated sampling. The middle panels illustrate the distribution of steps and generated samples. The right column demonstrates evolving intermediate predictions, showing uncertainty estimates and corresponding $x_0$ predictions.}
    \label{fig:sampling_progression}
        \vspace{-1em}
\end{figure}

\section{Conclusion}

We introduced AsyncPatch Diffusion, a spatially flexible joint diffusion framework that assigns different noise levels to different image regions while preserving joint denoising. We theoretically justify this formulation by showing that asynchronous corruption recovers the correct joint score and admits a valid ELBO.
Empirically, AsyncPatch preserves strong unconditional generation while enabling zero-shot inpainting without task-specific fine-tuning. Additionally, we introduce input guidance, which enables the generative model to more closely follow the known parts of the image and preserve texture even in artistic images. \\
Future work should study distillation methods that exploit spatially variable computation and investigate the interaction between asynchronous schedules and modern transformer-based diffusion architectures, higher-resolution generation, and text-conditioned models. Overall, AsyncPatch establishes spatially adaptive denoising as a simple mechanism for unifying standard generation, conditional reconstruction, and adaptive sampling within one diffusion model.

\newpage

\newpage
\bibliographystyle{plainnat}
\bibliography{local_references}
\newpage

\appendix
\input{supplementary_material}

\end{document}

%% file: supplementary_material.tex
\section{Proofs of the Lemmas}
\label{app:proof_lemmas}
\subsection{Proof \autoref{lem:joint-opt}}
Before we begin the proof, we must setup the basic definition and a preliminary Lemma.

\begin{definitionbox}{Generalized Denoising Score Matching (GDSM) Objective}
For a system with decoupled time variables, the generalized denoising score matching
(GDSM) objective is defined as
\begin{equation}
\mathcal{I}_{\mathrm{GDSM}}
=
\mathbb{E}_{\mathbf{t},\,\mathbf{x},\,\mathbf{z}_{\mathbf{t}}}
\left[
\sum_{i=1}^N
\frac{
\mathcal{L}_{i}(q_{\mathbf{t}\mid\mathbf{0}}/\beta_\theta)(\mathbf{z}_{i, t_i}, t_i)
}{
(q_{\mathbf{t}\mid\mathbf{0}}/\beta_\theta)(\mathbf{z}_{i, t_i}, t_i)
}
-
\mathcal{L}_{i}\log(q_{\mathbf{t}\mid\mathbf{0}}/\beta_\theta)(\mathbf{z}_{i, t_i}, t_i)
\right].
\end{equation}
Here, $\mathcal{L}_i$ denotes the infinitesimal generator of the forward process
for the $i$-th token, and $ q_{\mathbf{t} | \mathbf{0}} = q(\mathbf{z}_{\mathbf{t}} | \mathbf{x}) = \prod_{i=1}^{N} q_{t_i | 0} $ is the forward diffusion transition probability from
$\mathbf{0}=(0,0,\dots,0)$ to $\mathbf{t}=(t_1,t_2,\dots,t_N)$.
Moreover, $\beta_\theta=\beta_\theta(\mathbf{z}_{\mathbf{t}},\mathbf{t})$
is the model approximation to the \textit{unnormalized} underlying true joint distribution
$p(\mathbf{z}_{\mathbf{t}})$.
\end{definitionbox}

\begin{lemma*}{Minimizer equivalence}
Any model \(\beta_\theta\) satisfying
\[
\beta_\theta(\mathbf z_{\mathbf t}, \mathbf t)\propto p(\mathbf z_{\mathbf t}, \mathbf t)
\]
is a global minimizer of \(\mathcal I_{\mathrm{GDSM}}\).
\end{lemma*}
\begin{proofbox}[Proof of the minimizer equivalence lemma.]
By Theorem 2 of~\citet{rojas_diffuse_2025}, there exists a constant \(C\),
independent of \(\theta\), such that
\[
\mathcal I_{\mathrm{GDSM}}(\theta)
=
\mathcal I_{\mathrm{GESM}}(\theta)+C.
\]
Therefore,
\[
\operatorname*{arg\,min}_\theta \mathcal I_{\mathrm{GDSM}}(\theta)
=
\operatorname*{arg\,min}_\theta \mathcal I_{\mathrm{GESM}}(\theta).
\]

Now, by Theorem 1 of~\citet{rojas_diffuse_2025},
\[
\mathcal I_{\mathrm{GESM}}(\theta)\ge 0,
\]
with equality if and only if
\[
\beta_\theta(\mathbf z_{\mathbf t},\mathbf t)\propto p(\mathbf z_{\mathbf t},\mathbf t).
\]
Hence \(\mathcal I_{\mathrm{GDSM}}(\theta)\ge C\), and equality is attained exactly when
\[
\beta_\theta(\mathbf z_{\mathbf t},\mathbf t)\propto p(\mathbf z_{\mathbf t},\mathbf t).
\]
Therefore, any such \(\beta_\theta\) is a global minimizer of \(\mathcal I_{\mathrm{GDSM}}\).
\end{proofbox}

Now we are equipped to prove the following.

\begin{lemma}{Joint Optimization of Independent Tokens}{}
    Optimizing the objective $\mathcal{L}(\theta)$ with independent time-steps
    for each token maximizes a variational lower bound on the
    joint log-likelihood $\log p(\mathbf{x}_1, \dots, \mathbf{x}_N)$.
\end{lemma}

\begin{proofbox}[Proof of the Joint Optimization of Independent Tokens Lemma.]
Let $\mathbf{x}=(\mathbf{x}_1,\dots,\mathbf{x}_N)$, let
$\mathbf{t}=(t_1,\dots,t_N)$, and let
$\mathbf{z_t}=(\mathbf{z}_1,\dots,\mathbf{z}_N)$
denote the independently noised tokens, with
\begin{equation*}
q(\mathbf{z_t}| \mathbf{x}, \mathbf{t}) = \prod_{i=1}^N q(\mathbf{z}_{i, t_i}\mid \mathbf{x}_i,t_i),
\qquad
q(\mathbf{z}_{i, t_i}\mid \mathbf{x}_i,t_i) = \mathcal N(\alpha_{t_i}\mathbf{x}_i,\sigma_{t_i}^2 I).
\end{equation*}

We first show that $\mathcal{L}(\theta)$ is the Gaussian specialization of the
joint denoising score-matching objective.
For the $i$-th token, the forward process is an It\^o diffusion:
\[
d\mathbf z_{i, t_i} = f(\mathbf z_{i, t_i},t_i)\,dt_i + g(t_i)\,d\mathbf W_i,
\]
whose spatial generator is
\[
\widehat{\mathcal L}_i
=
f(\mathbf z_{i, t_i},t_i)\cdot \nabla_{\mathbf z_{i, t_i}}
+\frac12 g(t_i)^2 \Delta_{\mathbf z_{i, t_i}}.
\]
Therefore, the generator of the augmented process $(\mathbf z_{i, t_i},t_i)$ is
\[
\mathcal L_i h
=
\partial_{t_i} h
+
f(\mathbf z_{i, t_i},t_i)\cdot \nabla_{\mathbf z_{i, t_i}} h
+
\frac12 g(t_i)^2 \Delta_{\mathbf z_{i, t_i}} h.
\]
For the \(i\)-th token, define the score-matching operator associated with the generator \(\mathcal L_i\) by
\[
\Phi_i(h)
:=
\frac{\mathcal L_i h}{h}
-
\mathcal L_i \log h.
\]
This is the tokenwise analogue of the generalized score-matching operator
\(\Phi(f)=f^{-1}Lf-L\log f\) introduced in~\citet{rojas_diffuse_2025}. Substituting the augmented Gaussian generator and evaluating at
\[
h_i(\mathbf z_{\mathbf t},\mathbf t)
=
\frac{q(\mathbf z_{i, t_i}\mid \mathbf x_i,t_i)}{\beta_\theta(\mathbf z_{\mathbf t},\mathbf t)},
\]
yields
\[
\Phi_i(h_i)
=
\frac12 g(t_i)^2
\left\|
\nabla_{\mathbf z_{i, t_i}}\log \beta_\theta(\mathbf z_{\mathbf t},\mathbf t)
-
\nabla_{\mathbf z_{i, t_i}}\log q(\mathbf z_{i, t_i}\mid \mathbf x_i,t_i)
\right\|^2.
\]
Because the forward process factorizes over tokens conditional on
\(\mathbf x\), the full objective decomposes as
\[
I_{\mathrm{GDSM}}(\theta)
\propto
\sum_{i=1}^N
\mathbb E
\left[
g(t_i)^2
\left\|
\mathbf s_{\theta,i}(\mathbf z_{\mathbf t},\mathbf t)
-
\nabla_{\mathbf z_{i, t_i}}\log q(\mathbf z_{i, t_i}\mid \mathbf x_i,t_i)
\right\|^2
\right].
\]
This is exactly the continuous analogue of the multimodal denoising objective
derived in \citet{rojas_diffuse_2025}, where the joint score can be learned by
matching unimodal conditional scores under independently injected noise.

Next, since
\[
q(\mathbf z_{i, t_i}\mid \mathbf x_i,t_i)
=
\mathcal N(\alpha_{t_i}\mathbf x_i,\sigma_{t_i}^2 I),
\]
its score is
\[
\nabla_{\mathbf z_{i, t_i}}\log q(\mathbf z_{i, t_i}\mid \mathbf x_i,t_i)
=
-
\frac{\mathbf z_{i, t_i}-\alpha_{t_i}\mathbf x_i}{\sigma_{t_i}^2}.
\]
Parameterizing the score network through \(x\)-prediction as
\[
\mathbf s_{\theta,i}(\mathbf z_{\mathbf t},\mathbf t)
=
-
\frac{\mathbf z_{i, t_i}-\alpha_{t_i}\hat{\mathbf x}_{\theta,i}(\mathbf z_{\mathbf t},\mathbf t)}
{\sigma_{t_i}^2},
\]
we obtain
\[
\left\|
\mathbf s_{\theta,i}(\mathbf z_{\mathbf t},\mathbf t)
-
\nabla_{\mathbf z_{i, t_i}}\log q(\mathbf z_{i, t_i}\mid \mathbf x_i,t_i)
\right\|^2
=
\frac{\alpha_{t_i}^2}{\sigma_{t_i}^4}
\left\|
\hat{\mathbf x}_{\theta,i}(\mathbf z_{\mathbf t},\mathbf t)-\mathbf x_i
\right\|^2.
\]
Hence
\[
I_{\mathrm{GDSM}}(\theta)
\propto
\sum_{i=1}^N
\mathbb E\!\left[
\frac{g(t_i)^2\alpha_{t_i}^2}{\sigma_{t_i}^4}
\left\|
\hat{\mathbf x}_{\theta,i}(\mathbf z_{\mathbf t},\mathbf t)-\mathbf x_i
\right\|^2
\right]
=
\mathcal L(\theta),
\]
up to a \(\theta\)-independent multiplicative constant.

Since \(\mathcal L(\theta)\) is equivalent to \(\mathcal I_{\mathrm{GDSM}}(\theta)\), any global minimizer
\(\theta^\star\) of \(\mathcal L(\theta)\) is also a global minimizer of
\(\mathcal I_{\mathrm{GDSM}}(\theta)\). By the Minimizer Equivalence Lemma,
\[
\beta_{\theta^\star}(\mathbf z_{\mathbf t},\mathbf t)\propto p(\mathbf z_{\mathbf t},\mathbf t).
\]
Therefore,
\[
\mathbf s_{\theta^\star,i}(\mathbf z_{\mathbf t},\mathbf t)
=
\nabla_{\mathbf z_{i, t_i}}\log \beta_{\theta^\star}(\mathbf z_{\mathbf t},\mathbf t)
=
\nabla_{\mathbf z_{i, t_i}}\log p(\mathbf z_{\mathbf t},\mathbf t),
\qquad i=1,\dots,N,
\]
where the last equality holds because the multiplicative constant vanishes under
\(\nabla \log\). Hence \(\theta^\star\) recovers the correct joint score field for the
forward noising process. Finally, the reverse process constructed from this score
recovers the time-reversed forward marginals, so that at zero noise its terminal
marginal is exactly the data distribution \(p(\mathbf x)\). This proves the claim.
\end{proofbox}

\begin{lemma}{Joint optimization with correlated time variables}{lem:joint-opt-correlated}
Let \(\rho(\mathbf t)\) be any joint distribution over time vectors \(\mathbf t=(t_1,\dots,t_N)\), not necessarily factorized. Assume that, conditional on \((\mathbf x,\mathbf t)\), the forward corruption process factorizes as
\[
q(\mathbf z_{\mathbf t}\mid \mathbf x,\mathbf t)
=
\prod_{i=1}^N q(\mathbf z_{i, t_i}\mid \mathbf x_i,t_i),
\]
with each \(q(\mathbf z_{i, t_i}\mid \mathbf x_i,t_i)\) Gaussian. Define
\[
\mathcal L_\rho(\theta)
=
\mathbb E_{\mathbf x\sim p,\ \mathbf t\sim \rho,\ \mathbf z_{\mathbf t}\sim q(\cdot\mid \mathbf x,\mathbf t)}
\left[
\sum_{i=1}^N
\lambda(t_i)
\left\|
\nabla_{\mathbf z_{i, t_i}}\log q(\mathbf z_{i, t_i}\mid \mathbf x_i,t_i)
-
\mathbf s_\theta(\mathbf z_{\mathbf t},\mathbf t)_i
\right\|^2
\right].
\]
Then any global minimizer \(\theta^\star\) of \(\mathcal L_\rho(\theta)\) recovers the correct joint score
\[
\mathbf s_{\theta^\star}(\mathbf z_{\mathbf t},\mathbf t)_i
=
\nabla_{\mathbf z_{i, t_i}}\log p(\mathbf z_{\mathbf t},\mathbf t),
\qquad
\forall i,
\]
for \((\mathbf z_{\mathbf t},\mathbf t)\) in the support of \(\rho(\mathbf t)\). Consequently, the corresponding reverse process recovers \(p(\mathbf x)\) at zero noise.
\end{lemma}
\begin{proofbox}[Proof of the joint optimization with correlated time variables Lemma.]
Fix a joint time distribution \(\rho(\mathbf t)\). The only difference from the
independent-time setting is the sampling law of \(\mathbf t\); conditional on
\((\mathbf x,\mathbf t)\), the forward corruption still factorizes:
\[
q(\mathbf z_{\mathbf t}\mid \mathbf x,\mathbf t)
=
\prod_{i=1}^N q(\mathbf z_{i, t_i}\mid \mathbf x_i,t_i).
\]
Hence, for every fixed \(\mathbf t\), the same Gaussian specialization of the
GDSM objective used in the proof of Lemma~\ref{lem:joint-opt} gives
\[
\mathcal I_{\mathrm{GDSM},\rho}(\theta)
\propto
\mathbb E_{\mathbf x,\mathbf t,\mathbf z_{\mathbf t}}
\left[
\sum_{i=1}^N
\lambda(t_i)
\left\|
\mathbf s_{\theta}(\mathbf z_{\mathbf t},\mathbf t)_i
-
\nabla_{\mathbf z_{i, t_i}}\log q(\mathbf z_{i, t_i}\mid \mathbf x_i,t_i)
\right\|^2
\right],
\]
where the expectation over \(\mathbf t\) is now taken with respect to
\(\rho(\mathbf t)\). Since \(\lambda(t_i)>0\) on the relevant time support, this
positive reweighting does not change the set of global minimizers. Thus
\(\mathcal L_\rho(\theta)\) has the same global minimizers as the corresponding
\(\rho\)-weighted GDSM objective on \(\operatorname{supp}(\rho)\).

By the Minimizer Equivalence Lemma, any global minimizer satisfies
\[
\beta_{\theta^\star}(\mathbf z_{\mathbf t},\mathbf t)
\propto
p(\mathbf z_{\mathbf t},\mathbf t)
\]
for \((\mathbf z_{\mathbf t},\mathbf t)\) in the support of the training
distribution. Therefore, taking \(\nabla_{\mathbf z_{i, t_i}}\log\) on both sides gives
\[
\nabla_{\mathbf z_{i, t_i}}\log \beta_{\theta^\star}(\mathbf z_{\mathbf t},\mathbf t)
=
\nabla_{\mathbf z_{i, t_i}}\log p(\mathbf z_{\mathbf t},\mathbf t),
\qquad i=1,\dots,N,
\]
because the proportionality constant is independent of \(\mathbf z_{\mathbf t}\).
Under the score parameterization
\[
\mathbf s_{\theta^\star}(\mathbf z_{\mathbf t},\mathbf t)_i
=
\nabla_{\mathbf z_{i, t_i}}\log \beta_{\theta^\star}(\mathbf z_{\mathbf t},\mathbf t),
\]
we conclude that
\[
\mathbf s_{\theta^\star}(\mathbf z_{\mathbf t},\mathbf t)_i
=
\nabla_{\mathbf z_{i, t_i}}\log p(\mathbf z_{\mathbf t},\mathbf t),
\qquad
\forall i,
\]
for all \((\mathbf z_{\mathbf t},\mathbf t)\) in the support of
\(\rho(\mathbf t)\).

Finally, using this exact joint score in the reverse dynamics gives the
time-reversal of the forward noising process on the same time support. Hence the
reverse process recovers the forward marginals in reverse order, and at
\(\mathbf t=\mathbf 0\) its marginal is the data distribution
\(p(\mathbf x_1,\dots,\mathbf x_N)\). This proves the claim.
\end{proofbox}

It is important to note that the choice between $\mathbf{x}$-prediction, noise-prediction, or score-matching is largely a matter of training optimization and numerical stability, and are mathematically equivalent up to a re-weighting of the per-noise level objectives, provided the model has infinite capacity. For the purposes of this paper, we treat these formulations as equivalent representations of the same underlying generative framework. The primary challenge we address is not the choice of objective, but rather the increased complexity of the optimization space introduced by the decoupled time parameters $\mathbf{t}$.

\section{Proof of \autoref{thm:elbo}}
\label{sec:proof_of_elbo}

\newcommand{\vareps}{\varepsilon}
\newcommand{\valentin}[1]{\textcolor{red}{\textbf{Valentin: #1}}}
\newcommand{\bft}{\mathbf{t}}
\newcommand{\bfs}{\mathbf{s}}
  \newcommand{\nset}{\mathbb{N}}
  \newcommand{\rset}{\mathbb{R}}
  \newcommand{\msx}{\mathsf{X}}
  \newcommand{\bfy}{\mathbf{y}}
  \newcommand{\bfalpha}{\boldsymbol{\alpha}}
  \newcommand{\bfu}{\mathbf{u}}
  \newcommand{\bfh}{\mathbf{h}}
  \newcommand{\bfq}{\mathbf{q}}

In this section, we are going to prove \autoref{thm:elbo}.
We first give an outline of the proof of the result.
First, we recall that in the case of DDPM \cite{ho2020denoising,huang2021variational}, we have that there exist $w: \ [0,1] \to (0,+\infty)$ such that
\begin{equation}
    \mathbb{E}_p[\log p_\theta(\mathbf{x})] \geq \int_{[0,1]} w(t) \left\| \nabla_{\mathbf{z}_{t}} \log q(\mathbf{z}_{t} | \mathbf{x}, t) - \mathbf{s}_{\theta}(\mathbf{z}_{t}; t) \right\|^2 \rmd t + C , \label{eq:elbo_ddpm}
\end{equation}
wiht $C \in \rset$ which doesn't depend on $\theta$. Note that in the rest of this section, we denote $C \in \rset$ any constant which does not depend on $\theta$ and on the problem parameters. In particular, we will not track the particular value of $C$.
Now, we recall the main result of \autoref{thm:elbo}
    \begin{equation}
        \mathbb{E}_p[\log p_\theta(\mathbf{x})] \geq \sum_{i=1}^N \int_{[0,1]^N} \lambda_i(\mathbf{t}) \left\| \nabla_{\mathbf{z}_{i, t_i}} \log q(\mathbf{z}_{i, t_i} | \mathbf{x}_i, t_i) - \mathbf{s}_{\theta}(\mathbf{z}_{\mathbf{t}}; \mathbf{t})_i \right\|^2 \rmd \mathbf{t} + C . \label{eq:elbo_appendix}
    \end{equation}
    Therefore, \eqref{eq:elbo_appendix} is the direct expansion \eqref{eq:elbo_ddpm} to the multi-index setting.
    One way of obtaining \eqref{eq:elbo_ddpm} is to consider the discrete time setting
    \begin{equation}
        \mathbb{E}_p[\log p_\theta(\mathbf{x})] \geq \sum_{k=0}^K \lambda_k \left\| \nabla_{\mathbf{z}_{t_k}} \log q(\mathbf{z}_{t_k} | \mathbf{x}, t_k) - \mathbf{s}_{\theta}(\mathbf{z}_{t_k}; t_k) \right\|^2 - \mathbb{E}_p[\log p_\theta(\mathbf{x}| \mathbf{z}_{t_0})] + C , \label{eq:discrete_elbo_ddpm}
    \end{equation}
    where $\{t_k\}_{k=0}^K$ is a sequence of steps such that $\vareps = t_0 < t_k < t_{k+1} < t_K = 1-\vareps$, where $\vareps > 0$ and $|t_{k+1} -t_k| < 1/K$.

    By defining such a family and letting $K \to + \infty$ and $\vareps \to 0$ we recover \eqref{eq:elbo_ddpm} under regular smoothness assumptions.
    The collection $\{t_k\}_{k=0}^K$ can be interpreted as a \emph{monotonic random walk} on itself, i.e., the grid $\{t_k\}_{k=0}^K$. Of course, in that setting the random walk is monotonically increasing since the grid is uni-dimensional.

    The main idea is to notice that we can extend \eqref{eq:discrete_elbo_ddpm} to arbitrary paths. Namely we have
        \begin{equation}
        \mathbb{E}_p[\log p_\theta(\mathbf{x})] \geq \sum_{k=1}^K \lambda_{\bft_k} \left\| \nabla_{\mathbf{z}_{\mathbf{t}_k}} \log q(\mathbf{z}_{\mathbf{t}_k} | \mathbf{x}, \mathbf{t}_k) - \mathbf{s}_{\theta}(\mathbf{z}_{\mathbf{t}_k}; \bft_k) \right\|^2 - \mathbb{E}_p[\log p_\theta(\mathbf{x}| \mathbf{z}_{t_0})] + C, \label{eq:discrete_elbo_ddpm_path}
    \end{equation}
    where $\{\bft_k\}_{k=0}^K$ is a sequence of steps such that $(\vareps, \dots, \vareps) = \bft_0 < \bft_k < \bft_{k+1} < \bft_K = (1-\vareps, \dots, 1-\vareps)$, where $\vareps > 0$ and $\bft_{k+1,i} = \bft_{k,i}$, for all $i \in \{1, \dots, N\}$, except for one $j \in \{1, \dots, N\}$ such that $\bft_{k+1,j} = \bft_{k,j} + (1- 2\vareps)/K$.

    So the path $\{\bft_k\}_{k=0}^K$ can be seen as a monotonic random walk realization in the grid $(\{(1-2\vareps)k/K\}_{k=1}^K)^N$.
    Therefore taking the limit $K \to +\infty$ and $\vareps \to 0$ and averaging on possible monotonic random walks is a viable strategy to derive \eqref{eq:elbo_appendix}.

    In \autoref{sec:discrete_grid}, we introduce formally monotonic random walks on grids and show the obtained limit when $K \to +\infty$ for a family of random walks.
    In \autoref{sec:trajectory}, we show that the (deterministic) continuous limit of the monotonic random walk admits an explicit representation.
    In \autoref{sec:pushforward}, we derive a density on $[0,1]^N$ by averaging over those curves, by considering random walks in a random environment.
    We put everything together in \autoref{sec:conclusion_proof} where we conclude the proof of \autoref{thm:elbo}.

\subsection{Discrete grid, monotonic random walks and fluid limit}
\label{sec:discrete_grid}

  Our main goal in this section is to show that one can control monotonic random walks when the grid resolution $K \to +\infty$.

  \paragraph{Discrete grid and monotonic paths.}

We start by introducing a few useful definitions for our analysis.

  \begin{definition}{Discrete Grid}{}
  Let $K \in \nset$, with $K \ge 1$, and let $\vareps \in (0,1/2)$.
  The \emph{grid of resolution~$K$} is the lattice
  \[
    G_K \;=\; \Bigl\{\vareps + \frac{(1-2\vareps)\,k}{K}
      : k\in\{0,1,\dots,K\}\Bigr\}^N
    \;\subset\; [\vareps,1-\vareps]^N,
  \]
  where we denote
  $\boldsymbol{\vareps}=(\vareps,\dots,\vareps)$ and
  $\mathbf{1}-\boldsymbol{\vareps}=(1-\vareps,\dots,1-\vareps)$.
  \label{def:grid}
  \end{definition}

  Second, we define a monotonic path. This will be key to define the sequence of times later on, see \eqref{eq:discrete_elbo_ddpm_path} for a motivation of this definition.

  \begin{definition}{Monotonic path}{}
  A \emph{monotonic path of resolution~$K$} is a sequence
  $\gamma = (\gamma_0, \gamma_1, \dots, \gamma_{NK})$ of points in~$G_K$
  satisfying:
  \begin{enumerate}
    \item[(i)] $\gamma_0 = \boldsymbol{\vareps}$ and
      $\gamma_{NK} = \mathbf{1}-\boldsymbol{\vareps}$;
    \item[(ii)] for each $k\in\{0,\dots,NK-1\}$, there exists an index
      $c(k)\in\{1,\dots,N\}$ such that
      $\gamma_{k+1} = \gamma_k +
        \frac{1-2\vareps}{K}\,\mathbf{e}_{c(k)}$,
      where $\mathbf{e}_i$ denotes the $i$-th standard basis vector
      in~$\rset^N$.
  \end{enumerate}
  We denote $\mathcal{P}_K$ the set of all monotonic paths of
  resolution~$K$.
  \label{def:mono-path}
  \end{definition}

  \paragraph{The Markov chain.}\label{subsec:markov}

  We now define a probability measure on~$\mathcal{P}_K$ by generating
  the path step-by-step using a Markov chain.

  \begin{definition}{Monotonic random walk}{}
  Fix a vector of \emph{environment weights}
  $\boldsymbol{\alpha} = (\alpha_1,\dots,\alpha_N) \in (0,+\infty)^N$.
  The \emph{monotonic random walk with environment~$\boldsymbol{\alpha}$
  and resolution~$K$} is the discrete-time Markov chain
  $(\bft_k^{(K)})_{k=0}^{NK}$ on the state space~$G_K$, defined on a
  filtered probability space
  $(\Omega,\mathscr{F},(\mathscr{F}_k)_{k\ge 0},\mathbb{P})$, with:
  \begin{itemize}
    \item \textbf{Initial state.}
      $\bft_0^{(K)} = \boldsymbol{\vareps}$.
    \item \textbf{Transition rule.} At step~$k$, conditional on
      $\mathscr{F}_k$, we define
      $\bft_{k+1}^{(K)} = \bft_k^{(K)} +
      \frac{1-2\vareps}{K}\,\mathbf{e}_i$ with probability $p_i$ for any $i \in \{1, \dots, N\}$, where  $\{p_i\}_{i=1}^N$ is given for any $i \in \{1, \dots, N\}$ by
      \begin{equation}\label{eq:trans-prob}
        p_i(\bft_k^{(K)};\boldsymbol{\alpha})
        \;=\;
        \frac{\alpha_i\bigl(1-\vareps - (\bft_k^{(K)})_i\bigr)}
             {\displaystyle\sum_{j=1}^N
               \alpha_j\bigl(1-\vareps - (\bft_k^{(K)})_j\bigr)},
      \end{equation}
    \item \textbf{Terminal state.} At step~$NK$, each coordinate has been
      incremented exactly~$K$ times, so
      $\bft_{NK}^{(K)} = \mathbf{1}-\boldsymbol{\vareps}$.
  \end{itemize}
  \end{definition}

  The denominator in~\eqref{eq:trans-prob} is
  $S(\bft;\boldsymbol{\alpha})
  = \sum_{j=1}^N \alpha_j(1-\vareps-(\bft_k^{(K)})_j)$.  Since each $\alpha_j > 0$ and
  $(\bft_k^{(K)})_j \le 1-\vareps$ with equality only at the terminal state, $S > 0$
  at every step $k < NK$, ensuring the probabilities are well-defined
  throughout the walk.

  The choice of \eqref{eq:trans-prob} is arbitrary. In our case, we choose it since in our randomized environment setting it will yield a density on the whole hypercube $(0,1)^N$. However, other choices are possible for the definition of $\{p_i\}_{i=1}^N$. The question of the best choice of transition probability to obtain the  tightest ELBO is still open.

  To obtain results that are independent of a specific choice
  of~$\boldsymbol{\alpha}$, we endow the environment with a probability
  measure.

  \begin{definition}{Randomised environment}{}
  The \emph{randomised environment} is the random vector
  $\boldsymbol{\alpha} = (\alpha_1,\dots,\alpha_N)$, where the
  $\alpha_i$ are independent, identically distributed random variables
  with $\alpha_i \sim \mathrm{Exp}(1)$ for $i \in \{1, \dots, N\}$.
  \label{def:env}
  \end{definition}

  The choice of exponential weights is not arbitrary: it leads to a
  particularly clean pushforward density in
  \autoref{sec:density}.
  In the next paragraph however, we will consider a fluid limit result, i.e. letting $K \to +\infty$.

\paragraph{Convergence to fluid limit.}

  We embed the discrete walk in continuous time by setting
  $\mathbf{Y}^{(K)}(\tau) = \bft^{(K)}_{\lfloor K\tau\rfloor}$ for
  $\tau\in[0,N]$.  The fluid limit is the solution of the ordinary
  differential equation driven by the transition
  probabilities~\eqref{eq:trans-prob}.

  \begin{theorem}{Fluid limit}{}
  Let $\boldsymbol{\alpha}\in (0,+\infty)^N$ and let
  $\mathbf{p}(\,\cdot\,;\boldsymbol{\alpha}):
    [\vareps,1-\vareps]^N\to\Delta^{N-1}$
  be the transition probability vector defined
  in~\eqref{eq:trans-prob}, where we define $\Delta^{N-1}$ as the $N-1$ dimensional simplex.
  Then:
    \label{thm:fluid}
  \begin{enumerate}
    \item[\textup{(i)}] The initial-value problem
      \begin{equation}\label{eq:ode}
        \dot{\mathbf{y}}(\tau)
        = (1-2\vareps)\,
          \mathbf{p}\bigl(\mathbf{y}(\tau);\boldsymbol{\alpha}\bigr),
        \qquad \mathbf{y}(0)=\boldsymbol{\vareps},
      \end{equation}
      has a unique solution
      $\mathbf{y}:[0,N]\to[\vareps,1-\vareps)^N$.
    \item[\textup{(ii)}] The scaled walk converges uniformly in
      probability to this solution. More precisely, for every $\delta>0$,
      \[
        \lim_{K \to +\infty} \mathbb{P}\!\left(
          \sup_{0\le \tau\le N}
          \bigl\|\mathbf{Y}^{(K)}(\tau)-\mathbf{y}(\tau)\bigr\|
          > \delta
        \right) = 0
        .
      \]
  \end{enumerate}
  \end{theorem}

\begin{proofbox}[Proof.]
  We decompose the proof into four parts: the Doob decomposition of the
  walk, martingale concentration, the integral formulation with
  discretisation error, and the final Gronwall argument.
  We first let $K \in \nset$ with $K \geq 1$ be fixed

  \textbf{Doob decomposition.}
  For any $k \in \{0, \dots, KN\}$
  \[
    \bft_k^{(K)}
    \;=\; \boldsymbol{\vareps}
      \;+\; \mathbf{A}_k^{(K)}
      \;+\; \mathbf{M}_k^{(K)},
  \]
  where the \emph{predictable compensator} is defined as
  \begin{equation}\label{eq:compensator}
    \mathbf{A}_k^{(K)}
    \;=\;
    \frac{1-2\vareps}{K}\sum_{j=0}^{k-1}
      \mathbf{p}\bigl(\bft_j^{(K)};\boldsymbol{\alpha}\bigr),
  \end{equation}
  and $\mathbf{M}_k^{(K)} =
    \bft_k^{(K)} - \boldsymbol{\vareps} - \mathbf{A}_k^{(K)}$
  is an $\mathscr{F}_k$-martingale with
  $\mathbf{M}_0^{(K)}=\mathbf{0}$. For any $j \in \{0, \dots, KN-1\}$, the martingale difference at step~$j$ is given by
  \[
    \Delta\mathbf{M}_j^{(K)}
    \;=\;
    \bft_{j+1}^{(K)} - \bft_j^{(K)}
    - \frac{1-2\vareps}{K}\,
      \mathbf{p}\bigl(\bft_j^{(K)}\bigr).
  \]
  Since $\bft_{j+1}^{(K)} - \bft_j^{(K)}
  = \frac{1-2\vareps}{K}\,\mathbf{e}_i$ for some~$i \in \{1, \dots, N\}$ and
  $\|\mathbf{p}\|\le 1$, we have the deterministic bound for any $j \in \{0, \dots, KN-1\}$
  \begin{equation}\label{eq:mg-bound}
    \bigl\|\Delta\mathbf{M}_j^{(K)}\bigr\|
    \;\le\; \frac{2(1-2\vareps)}{K}.
  \end{equation}
  Using this result, we get that for any $j \in \{0, \dots, KN-1\}$
  \begin{equation}
    \mathbb{E}\!\left[
      \bigl\|\Delta\mathbf{M}_j^{(K)}\bigr\|^2
      \;\big|\;\mathscr{F}_j
    \right]
    \;\le\; \frac{4(1-2\vareps)^2}{K^2}. \label{eq:square_bound}
  \end{equation}

  \textbf{Martingale concentration.}
  Define the continuous-time martingale
  $\mathbf{Z}^{(K)}(\tau) = \mathbf{M}_{\lfloor K\tau\rfloor}^{(K)}$
  for $\tau\in[0,N]$.  Using \eqref{eq:square_bound}, we get that
  \[
    \mathbb{E}\!\left[
      \bigl\|\mathbf{Z}^{(K)}(N)\bigr\|^2
    \right]
    \;=\;
    \sum_{j=0}^{NK-1}
      \mathbb{E}\!\left[
        \bigl\|\Delta\mathbf{M}_j^{(K)}\bigr\|^2
      \right]
    \;\le\;
    NK \cdot \frac{4(1-2\vareps)^2}{K^2}
    \;=\; \frac{4N(1-2\vareps)^2}{K}.
  \]
  Combining this result and Doob's $L^2$ maximal inequality, see \citep{billingsley2013convergence} for instance, we then obtain
  \begin{equation}\label{eq:doob}
    \mathbb{E}\!\left[
      \sup_{0\le \tau\le N}
      \bigl\|\mathbf{Z}^{(K)}(\tau)\bigr\|^2
    \right]
    \;\le\;
    4\,\mathbb{E}\!\left[
      \bigl\|\mathbf{Z}^{(K)}(N)\bigr\|^2
    \right]
    \;\le\;
    \frac{16N(1-2\vareps)^2}{K}.
  \end{equation}
Combining this result and Markov's inequality, we get for any $\delta>0$,
  \begin{equation}
    \mathbb{P}\!\left(
      \sup_{0\le \tau\le N}
      \bigl\|\mathbf{Z}^{(K)}(\tau)\bigr\| > \delta
    \right)
    \;\le\;
    \frac{16N(1-2\vareps)^2}{\delta^2\,K}
    \;\xrightarrow{K\to\infty}\; 0.
    \label{eq:martingale_concentration}
  \end{equation}

  \textbf{Integral formulation and discretisation error.}
  The compensator~\eqref{eq:compensator} is a Riemann sum approximation
  to the integral
  $(1-2\vareps)\int_0^\tau
    \mathbf{p}(\mathbf{Y}^{(K)}(s))\,\rmd s$.
  More precisely, for any $\tau \in [0,N]$ we have
  \begin{equation}\label{eq:integral-form}
    \mathbf{Y}^{(K)}(\tau)
    \;=\;
    \boldsymbol{\vareps}
    \;+\; (1-2\vareps)\int_0^\tau
      \mathbf{p}\bigl(\mathbf{Y}^{(K)}(s)\bigr)\,\rmd s
    \;+\; \mathbf{Z}^{(K)}(\tau)
    \;+\; \mathbf{E}^{(K)}(\tau),
  \end{equation}
  where $\mathbf{E}^{(K)}(\tau)$ is a discretisation error arising from
  the piecewise-constant interpolation.  Since $\mathbf{p}$ is locally
  Lipschitz with constant~$L$ on any compact subset of
  $\msx = \{\bft\in\rset^N \ , \ \bft \in (0,+\infty)^N, \bft_i < 1- \vareps \text{for all $i \in \{1, \dots, N\}$}\}$, and the step size
  is~$(1-2\vareps)/K$, this error is bounded deterministically:
  \begin{equation}\label{eq:disc-error}
    \sup_{0\le \tau\le N}
    \bigl\|\mathbf{E}^{(K)}(\tau)\bigr\|
    \;\le\;
    \frac{C}{K},
  \end{equation}
  where $C$ depends on~$L$, $N$, and~$\vareps$ but not on the
  realisation of the walk.

  \textbf{Gronwall's lemma.}
  The fluid limit $\mathbf{y}(\tau)$ satisfies the integral equation
  $\mathbf{y}(\tau) = \boldsymbol{\vareps} +
    (1-2\vareps)\int_0^\tau
      \mathbf{p}(\mathbf{y}(s))\,\rmd s$,
  which has a unique solution using Picard theorem, since (the vector field
  $(1-2\vareps)\,\mathbf{p}$ is locally Lipschitz on $\msx$). Similarly, since the field is bounded on $\msx$, we have that $y(\tau)$ is well-defined for all $\tau \in [0,N]$.
  Combining the definition of $\bfy(\tau)$ and \eqref{eq:integral-form}, we obtain that for any $\tau \in [0,N]$
  \begin{align*}
    \bigl\|\mathbf{Y}^{(K)}(\tau) - \mathbf{y}(\tau)\bigr\|
    &\;\le\;
    (1-2\vareps)\int_0^\tau
      L\,\bigl\|\mathbf{Y}^{(K)}(s)-\mathbf{y}(s)\bigr\|\,\rmd s
    \;+\;
    \sup_{s\le N}\bigl\|\mathbf{Z}^{(K)}(s)\bigr\|
    \;+\; \frac{C}{K}.
  \end{align*}
  Define $\eta^{(K)} =
  \sup_{s\le N}\|\mathbf{Z}^{(K)}(s)\| + C/K$.  By Gronwall's
  inequality:
  \begin{equation}\label{eq:gronwall}
    \sup_{0\le \tau\le N}
    \bigl\|\mathbf{Y}^{(K)}(\tau)-\mathbf{y}(\tau)\bigr\|
    \;\le\;
    \eta^{(K)}\exp[(1-2\vareps)LN].
  \end{equation}
  Since $\eta^{(K)}\xrightarrow{\mathbb{P}} 0$ using \eqref{eq:martingale_concentration} and the deterministic bound~\eqref{eq:disc-error}, we complete the proof.
\end{proofbox}

\subsection{The Explicit Trajectory}\label{sec:trajectory}

We now solve the ODE~\eqref{eq:ode} in closed form.  The key idea is
an auxiliary time-change that decouples the $N$~coordinates, reducing a
coupled nonlinear system to $N$~independent linear ODEs.

\paragraph{Auxiliary time-change.}
Recall the ODE from \autoref{thm:fluid}. For any $i \in \{1, \dots, N\}$
\begin{equation}\label{eq:ode-recall}
  \frac{\rmd \bfy_i(\tau)}{\rmd \tau}
  \;=\;
  \frac{(1-2\vareps)\,\bfalpha_i(1-\vareps-\bfy_i(\tau))}{S(\tau)},
  \qquad
  S(\tau) = \sum_{j=1}^N \bfalpha_j\bigl(1-\vareps-\bfy_j(\tau)\bigr),
  \qquad
  \bfy_i(0)=\vareps.
\end{equation}
The difficulty is that the denominator~$S(\tau)$ couples all coordinates $\bfy_i$.
We eliminate this coupling by introducing a new time variable.

\begin{definition}{Auxiliary time}{}
Define the \emph{auxiliary time} $\sigma:[0,N)\to[0,\infty)$ such that for any $\tau \in [0,N]$ we have
\begin{equation}\label{eq:sigma-def}
  \frac{\rmd\sigma(\tau)}{\rmd \tau} = \frac{1-2\vareps}{S(\tau)},
  \qquad \sigma(0)=0.
\end{equation}
Since $S(\tau)>0$ for $\tau\in[0,N)$ (as long as
$\mathbf{y}(\tau)\ne\mathbf{1}-\boldsymbol{\vareps}$),
the function~$\sigma(\tau)$ is well-defined, strictly increasing, and
smooth on~$[0,N)$.
\label{def:aux-time}
\end{definition}

The intuition behind this change of variables is as follows.  The
denominator~$S(\tau)$ measures the total ``remaining capacity'' of the
walk.  When $S$ is large (early in the walk, when all coordinates are
far from~$1-\vareps$), the auxiliary clock runs slowly relative to
physical time.  When $S$ is small (near the end, when most coordinates
are close to~$1-\vareps$), the auxiliary clock accelerates.
In the following proposition, we show that this simple change of variable yield the announced decoupling.

\begin{proposition}{Decoupled dynamics}{}
Let $\bfy$ be given by \eqref{eq:ode-recall}.
We have that for any $i \in \{1,\dots, N\}$
\begin{equation}\label{eq:decoupled}
  \frac{\rmd \bfy_i}{\rmd\sigma} = \bfalpha_i(1-\vareps-\bfy_i),
  \qquad \bfy_i(0)=\vareps,
\end{equation}
which admits a unique solution given for any $i \in \{1,\dots,N\}$ by
\begin{equation}\label{eq:trajectory}
  \bfy_i(\sigma) = (1-\vareps) - (1-2\vareps)\,\exp[-\bfalpha_i \sigma],
\end{equation}
\label{prop:decoupled}
\end{proposition}

\begin{proofbox}[Proof.]
Using the chain rule and~\eqref{eq:sigma-def} we have that for any $i \in \{1, \dots, N\}$
\[
  \frac{\rmd \bfy_i}{\rmd\sigma}
  = \frac{\rmd \bfy_i}{\rmd \tau}\cdot\frac{\rmd \tau}{\rmd\sigma}
  = \frac{(1-2\vareps)\,\bfalpha_i(1-\vareps-\bfy_i)}{S}
    \cdot\frac{S}{1-2\vareps}
  = \bfalpha_i(1-\vareps-\bfy_i).
\]
This is a standard first-order linear ODE.  Setting
$\bfu_i = (1-\vareps)-\bfy_i$, we have $\dot{\bfu}_i = -\alpha_i \bfu_i$ with
$\bfu_i(0)=1-2\vareps$, hence
$\bfu_i(\sigma)=(1-2\vareps)\exp[-\bfalpha_i\sigma]$, which concludes the proof.
\end{proofbox}

\paragraph{A few properties.}
Before moving on to the next section and the randomized environment, we state a few properties of the time $\sigma$ and the curve $\bfy$.
The physical time~$\tau$ is recovered from the auxiliary time~$\sigma$
by observing that
$\tau = \frac{1}{1-2\vareps}\sum_{i=1}^N(\bfy_i - \vareps)$
(since each step of the discrete walk increases~$\tau$ by~$1/K$ and
one coordinate by~$(1-2\vareps)/K$, and noting that $\bfy_i$ refers to the $i$-th coordinate in the hypercube with dimensionality $N$).

\begin{proposition}{Physical--auxiliary time correspondence}{}
The physical time as a function of the auxiliary time is
\begin{equation}\label{eq:t-of-sigma}
  \tau(\sigma)
  = \frac{1}{1-2\vareps}\sum_{i=1}^N
    \bigl(\bfy_i(\sigma)-\vareps\bigr)
  = N - \sum_{i=1}^N \exp[-\bfalpha_i\,\sigma].
\end{equation}
In particular, the map $\sigma\mapsto \tau(\sigma)$ is a $C^\infty$-diffeomorphism
from $[0,\infty)$ onto $[0,N)$.
  \label{prop:time-map}
\end{proposition}

\begin{proof}
Equation~\eqref{eq:t-of-sigma} follows by
summing~\eqref{eq:trajectory} over~$i$. The map is strictly increasing (since $S>0$) with $\tau(0)=0$ and
$\tau(\sigma)\to N$ as $\sigma\to\infty$.  Smoothness and bijectivity
onto $[0,N)$ follow from the inverse function theorem.
\end{proof}

Finally, we highlight a few properties of the trajectory $\bfy$.
\begin{proposition}{Invariance and asymptotics}{}
The fluid limit trajectory $\bfy$ has the following properties:
\begin{enumerate}
  \item[\textup{(i)}] \textbf{Invariance of
    $[\vareps,1-\vareps)^N$}:
    For all $\sigma\in[0,\infty)$ and all $i=1,\dots,N$,
    \[
      \vareps \;\le\; \bfy_i(\sigma)
      = (1-\vareps) - (1-2\vareps)\,\exp[-\bfalpha_i\sigma]
      \;<\; 1-\vareps.
    \]
    The trajectory remains strictly
    inside~$[\vareps,1-\vareps)^N$ for all finite auxiliary time.

  \item[\textup{(ii)}] \textbf{Convergence to
    $\mathbf{1}-\boldsymbol{\vareps}$}:
    As $\sigma\to\infty$, $y_i(\sigma)\to 1-\vareps$ for every~$i$,
    and $\tau(\sigma)\to N$.  The trajectory reaches the
    corner~$\mathbf{1}-\boldsymbol{\vareps}$ asymptotically.

  \item[\textup{(iii)}] \textbf{Monotonicity}:
    Each coordinate $\sigma\mapsto \bfy_i(\sigma)$ is strictly increasing.
    The trajectory is a monotonic curve
    in~$[\vareps,1-\vareps]^N$ in the sense that it respects the
    componentwise partial order.

  \item[\textup{(iv)}] \textbf{Discrete walk reaches
    $\mathbf{1}-\boldsymbol{\vareps}$ exactly}:
    In the discrete walk with~$NK$ steps, each of
    size~$(1-2\vareps)/K$, the constraint
    $\sum_i(\bfy_i-\vareps) = N(1-2\vareps)$ at step~$NK$ together with
    $\bfy_i\le 1-\vareps$ forces $\bfy_i=1-\vareps$ for all~$i$.  Thus the
    terminal condition
    $\bft_{NK}^{(K)}=\mathbf{1}-\boldsymbol{\vareps}$ holds almost
    surely.
\end{enumerate}
  \label{prop:trajectory-props}
\end{proposition}

\begin{proof}
Parts (i)--(iii) are immediate from the
formula~\eqref{eq:trajectory} and the positivity of~$\alpha_i$.
Part~(iv) follows from the observation that $p_i(\bft)=0$ whenever
$t_i=1-\vareps$, so no coordinate can exceed~$1-\vareps$; combined
with
$\sum_i\bigl((\bft_{NK}^{(K)})_i-\vareps\bigr)
  = NK\cdot(1-2\vareps)/K = N(1-2\vareps)$,
this forces every coordinate to equal~$1-\vareps$.
\end{proof}

Along the fluid limit curve, the transition probabilities take a
particularly elegant form: the environment
weights~$\boldsymbol{\alpha}$ cancel, yielding a \emph{universal}
expression depending only on the current position~$\mathbf{y}$.

\begin{proposition}{Universal transition probabilities}{}
For any $i \in \{1, \dots, N\}$, define
$\bfh_i = \log\frac{1-2\vareps}{1-\vareps-\bfy_i}$ for
$\bfy_i\in(\vareps,1-\vareps)$.  Then, for any $i \in \{1, \dots, N\}$, we have
\begin{equation}\label{eq:universal-trans}
  p_i(\mathbf{y})
  \;=\;
  \frac{(1-\vareps-\bfy_i)\,\bfh_i}
       {\displaystyle\sum_{j=1}^N(1-\vareps-\bfy_j)\,\bfh_j}.
\end{equation}
  \label{prop:universal-trans}
\end{proposition}

\begin{proof}
From the trajectory formula~\eqref{eq:trajectory},
$1-\vareps-\bfy_i = (1-2\vareps)\exp[-\bfalpha_i\sigma]$, so
$\bfalpha_i = \bfh_i/\sigma$ where
$\bfh_i = \bfalpha_i\sigma
     = \log\frac{1-2\vareps}{1-\vareps-\bfy_i}$.
Combining this result with \eqref{eq:trans-prob}, we get that
\[
  p_i(\mathbf{y};\boldsymbol{\alpha})
  = \frac{(\bfh_i/\sigma)(1-\vareps-\bfy_i)}
         {\sum_j(\bfh_j/\sigma)(1-\vareps-\bfy_j)}
  = \frac{(1-\vareps-\bfy_i)\,\bfh_i}
         {\sum_j(1-\vareps-y_j)\,\bfh_j},
\]
which concludes the proof.
\end{proof}

\subsection{The Pushforward Density}\label{sec:density}\label{sec:pushforward}

We now execute the outer layer of the proof strategy: averaging over
the exponential environment.  The goal is to derive an explicit density
$q(\mathbf{y})$ on the interior of the
hypercube~$(\vareps,1-\vareps)^N$ such that, for any integrable
function~$g$,
\begin{equation}\label{eq:density-goal}
  \mathbb{E}_{\boldsymbol{\alpha}}\!\left[
    \int_0^N g\bigl(\mathbf{y}(\tau)\bigr)\,\rmd \tau
  \right]
  \;=\;
  \int_{(\vareps,1-\vareps)^N}
    g(\mathbf{y})\,q(\mathbf{y})\,\rmd\mathbf{y}.
\end{equation}
The density~$q$ captures how much ``time'' the fluid limit trajectory
spends near each point of the hypercube, averaged over all possible
environments.

\paragraph{Notation.}
We collect here the abbreviations that will appear throughout the
computation.  For $\mathbf{y}\in(\vareps,1-\vareps)^N$ and $i \in \{1, \dots, N\}$, define
\begin{equation}\label{eq:abbreviations}
  \begin{aligned}
    \bfh_i &= \log\frac{1-2\vareps}{1-\vareps-\bfy_i} > 0,
    &\qquad&\text{(per-coordinate log-odds),}\\[4pt]
    R(\bfy) &= \sum_{j=1}^N \bfh_j,
    &\qquad&\text{(total log-odds),}\\[4pt]
    L(\bfy) &= \sum_{j=1}^N (1-\vareps-\bfy_j)\,\bfh_j,
    &\qquad&\text{(weighted capacity),}\\[4pt]
    P(\bfy) &= \prod_{j=1}^N (1-\vareps-\bfy_j).
    &\qquad&\text{(residual volume).}
  \end{aligned}
\end{equation}
All four quantities are strictly positive
on~$(\vareps,1-\vareps)^N$.

\paragraph{Main result.}
In particular, we prove the following result.

\begin{theorem}{Pushforward density}{}
There exists $\bfq$ such that
for any $g$ with $\sup_{x \in (0,1)^N} |g(x)| < +\infty$, we have that
\begin{equation}\label{eq:density-goal-prove}
  \mathbb{E}_{\boldsymbol{\alpha}}\!\left[
    \int_0^N g\bigl(\mathbf{y}(\tau)\bigr)\,\rmd \tau
  \right]
  \;=\;
  \int_{(\vareps,1-\vareps)^N}
    g(\mathbf{y})\,q(\mathbf{y})\,\rmd\mathbf{y}.
\end{equation}
The density~$\bfq$ in~\eqref{eq:density-goal-prove} is given for any $\mathbf{y}\in(\vareps,1-\vareps)^N$ by
\begin{equation}\label{eq:density}
    q(\mathbf{y})
    \;=\;
    \frac{(N-1)!\;L(\bfy)}
         {(1-2\vareps)\,P(\bfy)\cdot R^N(\bfy)}
    \;=\;
    \frac{(N-1)!\;\displaystyle\sum_{j=1}^N
      (1-\vareps-\bfy_j)\log\frac{1-2\vareps}{1-\vareps-\bfy_j}}
         {(1-2\vareps)\;\displaystyle\prod_{j=1}^N(1-\vareps-\bfy_j)
          \;\cdot\;
          \Bigl(\sum_{j=1}^N
            \log\frac{1-2\vareps}{1-\vareps-\bfy_j}\Bigr)^{\!N}} .
\end{equation}
\label{thm:density}
\end{theorem}

\begin{proofbox}[Proof.]
We split the proof into five steps: writing the expectation as a
double integral, changing variables from~$\boldsymbol{\alpha}$
to~$\mathbf{y}$, expressing all terms in the new coordinates,
separating the~$\sigma$-integral, and evaluating it in closed form.

\textbf{Step 1: Write the expectation as a double integral.}
The environment has density
$f(\boldsymbol{\alpha}) = \exp[-\sum_j\bfalpha_j]$ on~$(0,+\infty)^N$.
Using the time-change
$\rmd \tau = S(\sigma)\,\rmd\sigma/(1-2\vareps)$ from
\autoref{sec:trajectory}:
\begin{equation}\label{eq:step1}
  \mathbb{E}_{\boldsymbol{\alpha}}\!\left[
    \int_0^N g(\mathbf{y}(\tau))\,\rmd \tau
  \right]
  =
  \int_{\rset_{>0}^N}\int_0^\infty
    g\bigl(\mathbf{y}(\boldsymbol{\alpha},\sigma)\bigr)\;
    \frac{S(\sigma)}{1-2\vareps}\exp\left[-\sum_j\bfalpha_j\right]\;
    \rmd\sigma\,\rmd\boldsymbol{\alpha}.
\end{equation}
Here $\mathbf{y}(\boldsymbol{\alpha},\sigma)$ denotes the trajectory
at auxiliary time~$\sigma$ in environment~$\boldsymbol{\alpha}$, and
$S(\sigma) = \sum_j\bfalpha_j(1-\vareps-\bfy_j)
           = \sum_j\bfalpha_j(1-2\vareps)\exp[-\bfalpha_j\sigma]$
is the total remaining weight.

\textbf{Step 2: Change of variables
$\boldsymbol{\alpha}\to\mathbf{y}$ at fixed~$\sigma$.}
For each fixed $\sigma>0$ and for any $i \in \{1, \dots, N\}$, the map
\[
  \Phi_\sigma:(0,+\infty)^N\to(\vareps,1-\vareps)^N,
  \qquad
  \bfalpha_i \;\mapsto\;
    \bfy_i = (1-\vareps) - (1-2\vareps)\exp[-\bfalpha_i\sigma],
\]
is a smooth diffeomorphism.  Its inverse is
$\bfalpha_i = \bfh_i/\sigma$, where
$\bfh_i = \log\frac{1-2\vareps}{1-\vareps-\bfy_i}$ as above.
The Jacobian matrix associated with $\Phi_{\sigma}$ is diagonal and for any $i \in \{1, \dots, N\}$, we have
\[
  \frac{\partial \bfy_i}{\partial\bfalpha_j}
  = \sigma(1-2\vareps)\exp[-\bfalpha_i\sigma]\updelta_{ij}
  = \sigma(1-\vareps-\bfy_i)\,\updelta_{ij},
\]
where we recall $\updelta_{ij}$ is the Dirac mass that is equal to one if $i=j$ and zero otherwise.
Therefore, the Jacobian determinant is given by
\begin{equation}\label{eq:jacobian}
  \left|\det\frac{\partial\mathbf{y}}
    {\partial\boldsymbol{\alpha}}\right|
  = \sigma^N\prod_{j=1}^N(1-\vareps-\bfy_j)
  = \sigma^N\,P(\bfy).
\end{equation}

\textbf{Step 3:
$(\mathbf{y},\sigma)$~coordinates transform.}
Using that $\bfalpha_j = \bfh_j/\sigma$, we have $\exp[-\sum_j\bfalpha_j]
  = \exp[-R(\bfy)/\sigma]$ and
\begin{equation}
  \label{eq:exp-term}
  S(\sigma)
  = \sum_j\frac{\bfh_j}{\sigma}\,(1-\vareps-\bfy_j)
  = \frac{L(\bfy)}{\sigma}.
\end{equation}

\textbf{Step 4: Substitute and separate the $\sigma$-integral.}
Replacing $\rmd\boldsymbol{\alpha}$ by
$\rmd\mathbf{y}/(\sigma^N P(\bfy))$
and using \eqref{eq:step1} and \eqref{eq:exp-term}, we obtain that
\begin{align}
  \mathbb{E}_{\boldsymbol{\alpha}}\!\left[
    \int_0^N g(\mathbf{y}(\tau))\,\rmd \tau
  \right]
  &= \int_0^\infty\int_{(\vareps,1-\vareps)^N}
     g(\mathbf{y})\;
     \frac{L(\bfy)}{(1-2\vareps)\,\sigma}\;
     \exp[-R(\bfy)/\sigma]\;
     \frac{1}{\sigma^N\,P(\bfy)}\;
     \rmd\mathbf{y}\,\rmd\sigma \notag\\
  &= \int_{(\vareps,1-\vareps)^N}
     g(\mathbf{y})\;\frac{L(\bfy)}{(1-2\vareps)\,P(\bfy)}
     \underbrace{\left[
       \int_0^\infty
       \frac{e^{-R(\bfy)/\sigma}}{\sigma^{N+1}}\,\rmd\sigma
     \right]}_{I(R)(\bfy)}
     \;\rmd\mathbf{y}.
  \label{eq:step4}
\end{align}
The inner integral $I(R)$ depends on~$\mathbf{y}$ only through the
total log-odds~$R = \sum_j h_j$.

\textbf{Step 5: Evaluate the $\sigma$-integral.}
Letting $u = R(\bfy)/\sigma$, so that $\sigma = R(\bfy)/u$ and
$\rmd\sigma = -R(\bfy)\,u^{-2}\,\rmd u$.  We obtain that
\begin{align}
  I(R)(\bfy)
  &= \int_0^\infty \exp[-u]\left(\frac{u}{R(\bfy)}\right)^{N+1}
     \frac{R(\bfy)}{u^2}\,\rmd u\\
  &= R^{-N}(\bfy) \int_0^\infty u^{N-1}\,\exp[-u]\,\rmd u \\
  &= \frac{\Gamma(N)}{R^N(\bfy)}
  = \frac{(N-1)!}{R^N(\bfy)}.
  \label{eq:gamma-integral}
\end{align}
Combining this result and \eqref{eq:step4}, we get \eqref{eq:density}, which concludes the proof.
\end{proofbox}

\subsection{Concluding the Proof of \autoref{thm:elbo}}
\label{sec:conclusion_proof}

We now assemble the results from
\autoref{sec:discrete_grid}--\autoref{sec:pushforward} to complete the
proof of \autoref{thm:elbo}.  We proceed in four steps: (i)~establish a
discrete ELBO along any monotonic path in the grid,
(ii)~average over the random walk and pass to the fluid limit,
(iii)~average over environments using the pushforward density, and
(iv)~identify the explicit weights~$\lambda_i$ and take~$\vareps\to 0$.

\paragraph{Notation.}
For $\bft\in(0,1)^N$ and $i\in\{1,\dots,N\}$, we write
\begin{equation}\label{eq:score-div}
  \mathcal{D}_i(\bft)
  \;=\;
  \bigl\|\nabla_{\mathbf{z}_{i,t_i}}
    \log q(\mathbf{z}_{i,t_i}\mid\mathbf{x}_i,t_i)
    - \mathbf{s}_\theta(\mathbf{z}_{\bft};\bft)_i\bigr\|^2
\end{equation}
for the per-coordinate score matching divergence at multi-index
time~$\bft$. We restate \autoref{thm:elbo} for convenience.

\begin{theorem}{Evidence Lower BOund, restated}{elbo-restated}
There exist $\{\lambda_i\}_{i=1}^N$ with
$\lambda_i:[0,1]^N\to(0,+\infty)$ for every
$i\in\{1,\dots,N\}$ such that
\begin{equation}\label{eq:elbo_restate}
  \mathbb{E}_p[\log p_\theta(\mathbf{x})]
  \;\ge\;
  \sum_{i=1}^N\int_{[0,1]^N}
    \lambda_i(\bft)\,\mathcal{D}_i(\bft)\,\rmd\bft + C ,
\end{equation}
where $C \geq 0$ is a constant which does not depend on $\theta$.
In addition, $\{\lambda_i\}_{i=1}^N$ is explicit and given
by~\eqref{eq:lambda_final}.
\end{theorem}

The proof relies on the following extension of the classical DDPM
variational bound to monotonic paths.

\begin{proposition}{Discrete multi-index ELBO}{}
Let $\vareps\in(0,1/2)$, $K\ge 1$, and let
$\gamma=(\gamma_0,\gamma_1,\dots,\gamma_{NK})\in\mathcal{P}_K$ be a
monotonic path in~$G_K$
(\autoref{def:mono-path}).  Then
\begin{equation}\label{eq:discrete_multi_elbo}
  \mathbb{E}_p[\log p_\theta(\mathbf{x})]
  \;\ge\;
  B_K(\vareps)
  \;+\;
  \sum_{k=0}^{NK-1}
    \frac{1-2\vareps}{K}\,
    \lambda\!\bigl((\gamma_k)_{c(k)}\bigr)\,
    \mathcal{D}_{c(k)}(\gamma_k),
\end{equation}
where $c(k)\in\{1,\dots,N\}$ is the coordinate updated at step~$k$,
$\lambda:(0,1)\to(0,+\infty)$ is the weight function from the
one-dimensional ELBO~\eqref{eq:elbo_ddpm}, and $B_K(\vareps)$
collects the boundary terms (reconstruction at
$\gamma_0=\boldsymbol{\vareps}$ and prior matching at
$\gamma_{NK}=\mathbf{1}-\boldsymbol{\vareps}$).
\label{prop:discrete_multi_elbo}
\end{proposition}

\begin{proofbox}[Proof.]
The argument is the standard DDPM variational decomposition
\cite{ho2020denoising,huang2021variational} applied along the
path~$\gamma$.  The key observation is that the forward process
factorises across coordinates:
$q(\mathbf{z}_{\bft}\mid\mathbf{x})
  =\prod_{i=1}^N q(z_{i,t_i}\mid x_i,t_i)$.
Along the monotonic path, each step~$k$ changes only coordinate~$c(k)$
from noise level~$(\gamma_k)_{c(k)}$ to
$(\gamma_{k+1})_{c(k)}=(\gamma_k)_{c(k)}+(1-2\vareps)/K$.
Therefore, the forward transition at step~$k$ is a one-dimensional
Gaussian transition for coordinate~$c(k)$ alone, and all other
coordinates are unchanged.

Applying Jensen's inequality with the forward process
$q(\mathbf{z}_{\gamma_0},\dots,\mathbf{z}_{\gamma_{NK}}\mid\mathbf{x})$
as the variational distribution and decomposing via telescoping KL
divergences yields one KL term per step.  Since step~$k$ is a
one-dimensional transition in coordinate~$c(k)$, the corresponding KL
divergence reduces to the score matching
term~$\frac{1-2\vareps}{K}\,\lambda((\gamma_k)_{c(k)})\,
\mathcal{D}_{c(k)}(\gamma_k)$
by the classical one-dimensional calculation, with weight
$\lambda$~inherited from~\eqref{eq:elbo_ddpm} and step
size~$(1-2\vareps)/K$.  The remaining reconstruction and prior terms
are collected in~$B_K(\vareps)$.
\end{proofbox}

\paragraph{Proof of \autoref{thm:elbo}.}
We now conclude with the proof of \autoref{thm:elbo}.
\begin{proofbox}[Proof.]
\textbf{Step~1: Average over the random walk and fluid limit.}
Fix $\vareps>0$ and $\boldsymbol{\alpha}\in(0,+\infty)^N$.  The
bound~\eqref{eq:discrete_multi_elbo} holds for every monotonic
path~$\gamma\in\mathcal{P}_K$ and, in particular, for the random path
generated by the monotonic random walk
$(\bft_k^{(K)})_{k=0}^{NK}$ with
environment~$\boldsymbol{\alpha}$.  Since the left-hand side is
deterministic, taking expectations over the walk preserves the
inequality. In particular, we have
\begin{equation}\label{eq:walk_avg}
  \mathbb{E}_p[\log p_\theta(\mathbf{x})]
  \;\ge\;
  B_K(\vareps)
  \;+\;
  \mathbb{E}_{\mathrm{walk}}\!\left[
    \sum_{k=0}^{NK-1}
      \frac{1-2\vareps}{K}\,
      \lambda\bigl((\bft_k^{(K)})_{c(k)}\bigr)\,
      \mathcal{D}_{c(k)}(\bft_k^{(K)})
  \right].
\end{equation}
At step~$k$, the direction $c(k)=i$ is chosen with
probability~$p_i(\bft_k^{(K)};\boldsymbol{\alpha})$
conditionally on~$\bft_k^{(K)}$.
By the tower property, the inner expectation decomposes as
\[
  \mathbb{E}_{\mathrm{walk}}\!\left[
    \sum_{k=0}^{NK-1}\frac{1}{K}\,\phi(\bft_k^{(K)})
  \right],
  \qquad
  \phi(\bft)
  \;=\;
  (1-2\vareps)\sum_{i=1}^N
    p_i(\bft;\boldsymbol{\alpha})\,
    \lambda(t_i)\,
    \mathcal{D}_i(\bft).
\]
By the fluid limit theorem (\autoref{thm:fluid}),
$\mathbf{Y}^{(K)}(\tau)=\bft_{\lfloor K\tau\rfloor}^{(K)}
  \to\mathbf{y}(\tau)$
uniformly in probability as $K\to\infty$.  Under continuity of
$\lambda$ and~$\mathcal{D}_i$, the Riemann sum converges and we have
\begin{equation}\label{eq:riemann_cv}
  \sum_{k=0}^{NK-1}\frac{1}{K}\,\phi(\bft_k^{(K)})
  \;\xrightarrow{K\to\infty}\;
  \int_0^N\phi\bigl(\mathbf{y}(\tau)\bigr)\,\rmd\tau
  \qquad\text{in probability.}
\end{equation}
Now we use a crucial property: by the universal transition
probabilities (\autoref{prop:universal-trans}), along the fluid limit
trajectory the transition probabilities $p_i(\mathbf{y}(\tau);
\boldsymbol{\alpha})$ are independent of~$\boldsymbol{\alpha}$ and
equal to
\[
  p_i(\mathbf{y})
  \;=\;
  \frac{(1-\vareps-\bfy_i)\,\bfh_i}
       {\sum_{j=1}^N(1-\vareps-\bfy_j)\,\bfh_j},
  \qquad
  \bfh_i = \log\frac{1-2\vareps}{1-\vareps-\bfy_i}.
\]
Therefore $\phi(\mathbf{y})$ is a function of~$\mathbf{y}$ alone, and
the bound becomes
\begin{equation}\label{eq:fluid_bound}
  \mathbb{E}_p[\log p_\theta(\mathbf{x})]
  \;\ge\;
  B(\vareps)
  \;+\;
  \int_0^N(1-2\vareps)
    \sum_{i=1}^N p_i(\mathbf{y}(\tau))\,
    \lambda(y_i(\tau))\,
    \mathcal{D}_i(\mathbf{y}(\tau))\,\rmd\tau,
\end{equation}
where $B(\vareps)=\lim_{K\to\infty}B_K(\vareps)$ and the
right-hand side depends on~$\boldsymbol{\alpha}$ only through the
trajectory~$\mathbf{y}(\tau)$.

\medskip

\textbf{Step~2: Average over environments.}
We now take expectations over
$\boldsymbol{\alpha}\sim\mathrm{Exp}(1)^{\otimes N}$.  The left-hand
side is unchanged.  Since $\phi(\mathbf{y})$ does not depend
on~$\boldsymbol{\alpha}$, the pushforward density theorem
(\autoref{thm:density}) applies directly:
\[
  \mathbb{E}_{\boldsymbol{\alpha}}\!\left[
    \int_0^N\phi(\mathbf{y}(\tau))\,\rmd\tau
  \right]
  \;=\;
  \int_{(\vareps,1-\vareps)^N}
    \phi(\mathbf{y})\,q(\mathbf{y})\,\rmd\mathbf{y},
\]
where $q$ is the pushforward density from \autoref{thm:density}.
Substituting the definition of~$\phi$ and exchanging the sum and
integral, we obtain
\begin{equation}\label{eq:env_avg}
  \mathbb{E}_p[\log p_\theta(\mathbf{x})]
  \;\ge\;
  B(\vareps)
  \;+\;
  \sum_{i=1}^N\int_{(\vareps,1-\vareps)^N}
    \underbrace{(1-2\vareps)\,q(\mathbf{y})\,p_i(\mathbf{y})\,
      \lambda(y_i)}_{\displaystyle\lambda_i^{\vareps}(\mathbf{y})}\;
    \mathcal{D}_i(\mathbf{y})\,\rmd\mathbf{y}.
\end{equation}

\medskip

\textbf{Step~3: Identify $\lambda_i$ and take $\vareps\to 0$.}
From~\eqref{eq:env_avg} we read off the weight at
$\vareps>0$:
\[
  \lambda_i^{\vareps}(\bft)
  \;=\;
  (1-2\vareps)\,q(\bft)\,p_i(\bft)\,\lambda(t_i).
\]
Substituting the explicit formulas from \autoref{thm:density} and
\autoref{prop:universal-trans}, the factor $(1-2\vareps)$ in the
numerator cancels the factor $(1-2\vareps)$ in the denominator of
the pushforward density, yielding
\[
  \lambda_i^{\vareps}(\bft)
  \;=\;
  \frac{(N-1)!\;\lambda(t_i)\;
    (1-\vareps-t_i)\,h_i}
       {P\cdot R^N},
\]
where $h_i$, $R$, $P$ are defined in~\eqref{eq:abbreviations}.
Taking $\vareps\to 0$, we get
\begin{equation}\label{eq:lambda_final}
  \lambda_i(\bft)
  \;=\;
  \frac{(N-1)!\;\lambda(\bft_i)\;
    (1-\bft_i)\bigl(-\log(1-\bft_i)\bigr)}
       {\displaystyle\prod_{j=1}^N(1-t_j)
        \;\cdot\;
        \Bigl(\sum_{j=1}^N\bigl(-\log(1-\bft_j)\bigr)\Bigr)^{\!N}},
  \qquad\bft\in(0,1)^N,
\end{equation}
where $\lambda:(0,1)\to(0,+\infty)$ is the weight function from the
one-dimensional ELBO~\eqref{eq:elbo_ddpm}.  Since
$\lambda(t_i)>0$, $(1-t_i)\in(0,1)$, and $-\log(1-t_i)>0$ for
$t_i\in(0,1)$, the weight~$\lambda_i(\bft)$ is strictly positive on
$(0,1)^N$, which completes the proof.
\end{proofbox}

\clearpage
\section{Timestep sampling during training}
\label{sec:appendix_timstep_sampling}
\begin{figure*}[t]
\centering
\small

\textbf{Perlin sampling}
\vspace{0.3em}

\begin{algorithmic}[1]
\State Sample Perlin mask $\mathbf m \in \{0,1\}^{H \times W}$
\State Sample $\mathbf t \sim \mathcal{U}(\mathbf 0, \mathbf 1)$ \Comment{Shape of $\mathbf{t}$ matches the shape of the mask and the shape of the image.}
\State Sample $b \sim \mathcal{B}(0,1)$  \Comment{Sample from Bernoulli distribution.}
\State \Return $\mathbf t\mathbf m + (1-\mathbf m) b$ \Comment{Background is either noisy or known.}
\end{algorithmic}

\vspace{0.8em}
\hrule
\vspace{0.8em}

\textbf{Patchwise sampling}
\vspace{0.3em}

\begin{algorithmic}[1]
\State Sample patch scale $K$  \Comment{Determines the number of patches.}
\State Partition image into patches $\{P_i\}_{i=1}^K$
\For{$i=1,\dots,K$}
    \State Sample $t_i \sim \mathcal{U}(0,1)$ \Comment{Sample for a single patch independently.}
    \State Set $\mathbf t|_{P_i} \gets t_i$ \Comment{Shape of $\mathbf{t}$ matches the shape of the mask and the shape of the image.}
\EndFor
\State \Return $\mathbf t$
\end{algorithmic}

\vspace{0.8em}
\hrule
\vspace{0.8em}

\textbf{AsyncPatch sampling}
\vspace{0.3em}

\begin{algorithmic}[1]
\State Sample $\bar{t} \sim \mathcal{U}(t_{\min},t_{\max})$ \Comment{Sample a mean value.}
\State Set $\delta \gets \min(\bar{t}-t_{\min},\, t_{\max}-\bar{t},\, 0.5)$ \Comment{Find the maximum possible delta for the timestep.}
\State Set $t^- \gets \bar{t}-\delta$, \quad $t^+ \gets \bar{t}+\delta$ \Comment{Extrema are symmetric around the mean.}
\State Sample patch scale $K$
\State Partition image into patches $\{P_i\}_{i=1}^K$
\For{$i=1,\dots,K$}
    \State Sample $t_i \sim \mathcal{U}(t^-,t^+)$  \Comment{This is guaranteed to have the chosen mean by construction.}
    \State Set $\mathbf t|_{P_i} \gets t_i$ \Comment{Shape of $\mathbf{t}$ matches the shape of the mask and the shape of the image.}
\EndFor
\State \Return $\mathbf t$
\end{algorithmic}

\caption{
Additional details on timestep sampling used during training.
}
\label{fig:long_timestep_sampling_algorithms_app}
\end{figure*}
Let the input be an image $\mathbf{x}\in \mathbb{R}^{H\times W}$, where $H$ is the height of the image, and $W$ is the width, and $\mathbf t = \{t_1, \dots, t_{H\times W}\}$ be the random variable that we use to sample the noise schedule for each pixel in the image during training.
Empirically, we find that effective timestep sampling methods are governed by three properties: the spatial scale over which timesteps are shared, the mean timestep $\bar{t} = \frac{1}{HW}\sum_{i} t_{i},$ and the spread $\Delta = \max_{i} t_{i} - \min_{i} t_{i}.$\\
The spatial scale controls locality, the mean controls the average corruption level, and the spread controls how often the model observes heterogeneous noise states with sharp transitions. Independent uniform sampling fails in both the average level of corruption, which concentrates near $0.5$, and the spatial scale, which is localized to a single pixel. This leads to a model that is poor at both traditional sampling, and realistic inpainting.

Although the forward process factorizes over tokens, the model does not denoise them independently in practice: each token is predicted only conditionally on the partially denoised image at some previous state $\mathbf{s} < \mathbf{t}$, where $t_i > s_i$ for all $i$ and $p_\theta(\mathbf{z}_{t_i}|\mathbf{z}_{\mathbf{s}}) \perp p_\theta(\mathbf{z}_{t_j}|\mathbf{z}_{\mathbf{s}})$ for $i \neq j$. The key quantity is therefore not only the timestep of each individual pixel, but also the total amount of information available in the image as a whole.

Under independent timestep sampling, the average timestep of an image concentrates sharply around $0.5$. In other words, the model is trained mostly on inputs where roughly half of the image information is available. Configurations in which large regions are jointly very noisy, or jointly nearly clean, are theoretically possible but extremely unlikely, because they require many pixels to share similar timesteps. Yet these more synchronized states are precisely the ones that matter during sampling, where generation may start from almost pure noise or where large spatial regions often evolve together. As a result, independent timestep sampling over-represents partially known images and under-represents the low-information and high-information regimes that are crucial at inference.

Using shared timesteps over patches partially alleviates this issue by introducing spatial correlation, but unless the patches become very large, the average timestep still remains concentrated near $0.5$. This suggests that what matters is not only local asynchrony, but exposing the model to a broad range of global information levels. We therefore sample timesteps in a controlled way by explicitly shaping both the mean timestep across the image and the spread of timesteps within it. The mean controls how much information is available overall, while the spread controls how synchronized or asynchronous the image is. Smaller spread values recover trajectories closer to traditional diffusion, whereas larger spread values expose the model to more spatially-flexible generation paths.

In Fig.~\ref{fig:long_timestep_sampling_algorithms_app} we report the sampling methods as described in the main text, with additional implementation details.

\section{Extended Related Work}
\label{app_related_work}

\paragraph{Multimodal diffusion and heterogeneous state spaces.}
Several recent works study diffusion or flow models with modality-dependent corruption schedules. \textit{UniDiffuser} \citep{bao_unidiffuser_2023}, \textit{MM-Diffusion} \citep{ruan2022mmdiffusion}, \textit{AVDiT} \citep{kim_avdit_2024}, \textit{OmniFlow} \citep{li_omniflow_2025}, and \textit{UniDisc} \citep{swerdlow2025unidisc} jointly model multiple modalities through decoupled noise processes. Many of these methods rely on latent-space alignment through VAEs, tokenizers, or pretrained encoders. More generally, \textit{GUD} \citep{gerdes_gud_2024} formulates diffusion with component-wise schedules in arbitrary bases, interpolating between autoregressive and diffusion-style generation.

A related direction considers multimodal generation directly on heterogeneous state spaces. Protein sequence--structure co-generation naturally couples discrete and geometric modalities, motivating methods such as \textit{MultiFlow} \citep{campbell_multiflow_2024} and \textit{Generator Matching} \citep{holderrieth_generator_2025}. The theoretical framework of \citet{rojas_diffuse_2025} further formalizes multimodal diffusion on product state spaces with decoupled time variables. Our work differs from these approaches by focusing on spatially heterogeneous diffusion schedules within a single image modality, where locality and partial observability become central challenges.

\paragraph{Diffusion inpainting and editing.}
Diffusion-based inpainting methods can broadly be divided into two categories. Zero-shot approaches adapt pretrained unconditional diffusion models during inference, typically by repeatedly enforcing consistency with observed pixels, as in \textit{RePaint} \citep{lugmayr_repaint_2022}. Task-specific approaches instead train or fine-tune dedicated inpainting architectures, including latent diffusion inpainting models \citep{rombach2022high} and modern diffusion editing systems such as \textit{GLIDE} and \textit{DiffEdit} \citep{nichol_glide_2022,couairon_diffedit_2022}. Recent editing-oriented methods such as \textit{GradPaint} further improve spatially controlled diffusion generation.

Most closely related to our work, \textit{RAD} \citep{kim_rad_2024} introduces pixel-level mask-dependent timesteps generated from Perlin noise. While effective for inpainting, RAD is specialized toward masked reconstruction. Our goal is instead to study general spatially heterogeneous timestep fields that support both localized reconstruction and unconstrained image generation.

\paragraph{Discussion on~\citet{schusterbauer2026patchforcing}}
During the final stages of this submission, we became aware of concurrent work by~\citet{schusterbauer2026patchforcing}, made public at the end of April 2026. Their work also studies spatially heterogeneous, token-wise denoising for image generation and identifies a similar train-test mismatch caused by naive independent timestep sampling. However, the two works differ in scope, formulation, and experimental emphasis.
\citet{schusterbauer2026patchforcing} focus on improving image generation quality through patch-wise denoising in DiT-based models. Their method uses fixed-size patches, introduces a learned patch-difficulty head, and designs difficulty-aware samplers that advance easier patches faster so that they can provide context to harder regions. In contrast, we study asynchronous diffusion as a general spatially flexible generative framework. We use a UNet-based implementation and sample timestep fields over varying spatial scales, from global-image schedules to patch-level and pixel-level schedules. This allows a single model to support not only standard image generation, but also zero-shot inpainting, autoregressive generation, texture synthesis, and uncertainty-guided acceleration.

The methods also differ in their treatment of timestep sampling. Patch Forcing controls the maximum available patch-level information during training, whereas in out work, we control both the average corruption level and the spatial variability of the timestep field. This design is intended to preserve synchronized diffusion trajectories while still exposing the model to heterogeneous local denoising states.

Finally, we introduce input guidance for conditioning on clean or partially corrupted regions, and provide a theoretical justification of asynchronous corruption through the joint-diffusion formulation and an ELBO derivation.

\section{Model architecture and compute}
\label{sec:model_architecture}

Here the details on the architecture used in the experiments. The UNet architecture is based on~\citet{dhariwal2021diffusion}, with modifications on the size depending on the experiments carried out (see Tab.~\ref{tab:unet-backbone-configs}).

Timestep conditioning is applied by simply changing the shape of the timestep to match the one of the input. This is possible because in the normal UNet-based architecture, conditioning is applied via FiLM modulation, with the same modulation applied to all elements in the feature map. Since the feature map already matches the shape of the input, we simply need to compute a different modulation per input dimension, which is straightforward by just feeding a timestep tensor with the data shape. Then, for the inner layers where downsampling is applied, we simply feed the downsampled version of the timestep to the FiLM modulation, thus matching the shape of the inner feature maps as well.

\subsection{Compute Resources}
\label{app:compute}

We report the approximate compute used for the main experiments and ablations in Table~\ref{tab:compute}. All runs were conducted on TPUv6 hardware. The total compute budget for the experiments reported in this work was approximately $2{,}024$ TPUv6-hours.

\begin{table}[h]
\centering
\caption{Approximate compute used for the main experiments and ablations.}
\label{tab:compute}
\begin{tabular}{lccc}
\toprule
Experiment group & Runs & Compute per run & Total compute \\
\midrule
ImageNet64 (600M model)
& 4
& $\sim 360$ TPUv6-hours
& $\sim 1{,}440$ TPUv6-hours \\

ImageNet256 LDM (600M model)
& 2
& $\sim 192$ TPUv6-hours
& $\sim 384$ TPUv6-hours \\

LSUN LDM  (600M model)
& 2
& $\sim 80$ TPUv6-hours
& $\sim 160$ TPUv6-hours \\

Ablations (60M model)
& 8
& $\sim 5$ TPUv6-hours
& $\sim 40$ TPUv6-hours \\
\midrule
\textbf{Total}
& --
& --
& $\sim 2{,}024$ TPUv6-hours \\
\bottomrule
\end{tabular}
\end{table}

For ImageNet64 experiments with the 600M-parameter model, each run took approximately 5 hours and 40 minutes on a \texttt{8x8} slice, corresponding to roughly 360 TPUv6-hours. The four main ImageNet64 experiments therefore required approximately 1,440 TPUv6-hours in total. For ImageNet256 latent diffusion model experiments, each run took approximately 3 hours on a \texttt{8x8} slice, corresponding to roughly 192 TPUv6-hours, for a total of 384 TPUv6-hours across the two main experiments. The LSUN latent diffusion model experiments required approximately 80 TPUv6-hours per run, or 160 TPUv6-hours in total. Finally, the ablations required approximately 5 TPUv6-hours each, with roughly 8 ablation runs for a total of 40 TPUv6-hours.

\paragraph{Training.}
The 60M ImageNet model is trained for $10^6$ steps with batch size $512$. We use a rectified-flow Gaussian diffusion process with velocity prediction and the SiD2 training objective.

The larger 600M ImageNet model is trained in the $64\times64\times3$ latent space of a pretrained VQ autoencoder in the case of ImageNet 256 and LSUN bedroom, and on the $64\times64\times3$ pixel input in the case of ImageNet 64. In the LDM case, the autoencoder is kept fixed, and only the diffusion U-Net is optimized. This model uses the same rectified-flow velocity-prediction objective as the 60M one.

Both models are trained using Adam optimization, gradient clipping, EMA with decay $0.999$ for the weights used in evaluation, and a learning-rate schedule with warmup.

\paragraph{Evaluation.}
Both models are evaluated using EMA weights and DDIM sampling with $250$ steps and stochastic coefficient of $0.25$.

\begin{table}[t]
\centering
\caption{U-Net backbone configurations. Parameter counts refer to the denoising network.}
\label{tab:unet-backbone-configs}
\begin{tabular}{lcc}
\toprule
\textbf{Hyperparameter} & \textbf{$\sim \mathbf{620}$M AsyncPatch} & \textbf{$\sim \mathbf{60}$M AsyncPatch} \\
\midrule
Base channels & $256$ & $128$ \\
Channel multipliers & $(1,2,3,4)$ & $(1,2,3,3)$ \\
Residual blocks & $(4,4,4,4)$ & $(1,2,2,2)$ \\
Middle blocks & $3$ & $1$ \\
\bottomrule
\end{tabular}
\end{table}

\section{Autoencoder Details}
\label{app:autoencoder}

\begin{table}[h]
\centering
\caption{Autoencoder architecture used for the ImageNet-$256^2$ and the LSUN experiments.}
\label{tab:autoencoder-architecture}
\begin{tabular}{lc}
\toprule
\textbf{Component} & \textbf{Configuration} \\
\midrule
Input resolution & $256 \times 256 \times 3$ \\
Latent resolution & $64 \times 64 \times 3$ \\
Downsampling factor & $4\times$ \\
Encoder channels & $128$ \\
Encoder channel multipliers & $(1,2,4)$ \\
Encoder residual blocks & $2$ per scale \\
Decoder channels & $128$ \\
Decoder channel multipliers & $(1,2,4)$ \\
Decoder residual blocks & $2$ per scale \\
Quantizer codebook size & $8192$ \\

\bottomrule
\end{tabular}
\end{table}

The ImageNet-$256^2$ latent diffusion model uses a pretrained VQ autoencoder to
map images to a spatial latent representation before diffusion. Images
$x \in \mathbb{R}^{256 \times 256 \times 3}$ are encoded into latents
$z \in \mathbb{R}^{64 \times 64 \times 3}$, giving a $4\times$ spatial
downsampling factor. Diffusion is performed entirely in this latent space, and
generated latents are decoded back to image space with the fixed decoder.

The autoencoder follows the same general U-Net design as the diffusion model:
it is built from convolutional residual blocks, multi-resolution downsampling
and upsampling paths, skip connections, RMS normalization, and optional
self-attention blocks. The main architectural difference is that the
autoencoder is asymmetric: the encoder maps the image to the latent grid,
whereas the decoder maps the quantized latent grid back to image space.

For masked image modeling and inpainting experiments, the mask is provided as
conditioning to both the encoder and the decoder. This conditioning uses the
same mechanism as timestep conditioning in the diffusion U-Net: the mask is
embedded and injected into the residual blocks through feature modulation,
rather than being used only as an extra input channel. In other words, the mask
acts as a conditioning signal that adaptively modulates the normalized features
throughout the autoencoder, analogously to how diffusion time modulates the
denoising network.

\section{Input Guidance vs. Classifier Free Guidance}
\begin{figure*}[htbp]
    \centering

    \begin{subfigure}[b]{\linewidth}
        \centering
        \begin{minipage}[b]{0.49\linewidth}
            \centering
            \includegraphics[width=\linewidth]{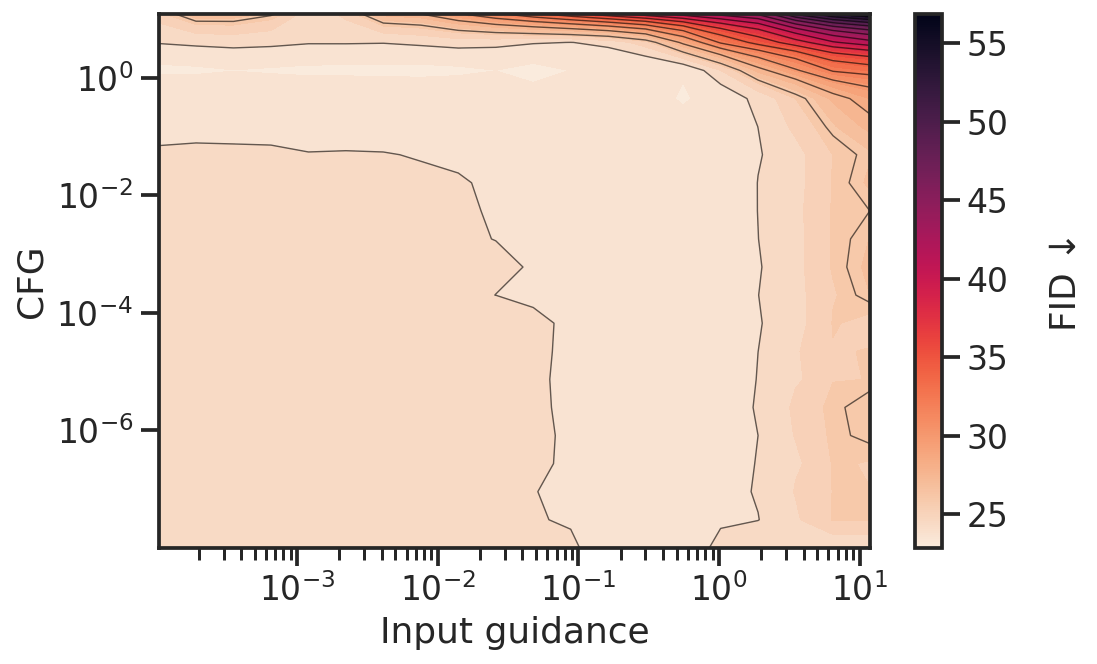}
        \end{minipage}
        \hfill
        \begin{minipage}[b]{0.49\linewidth}
            \centering
            \includegraphics[width=\linewidth]{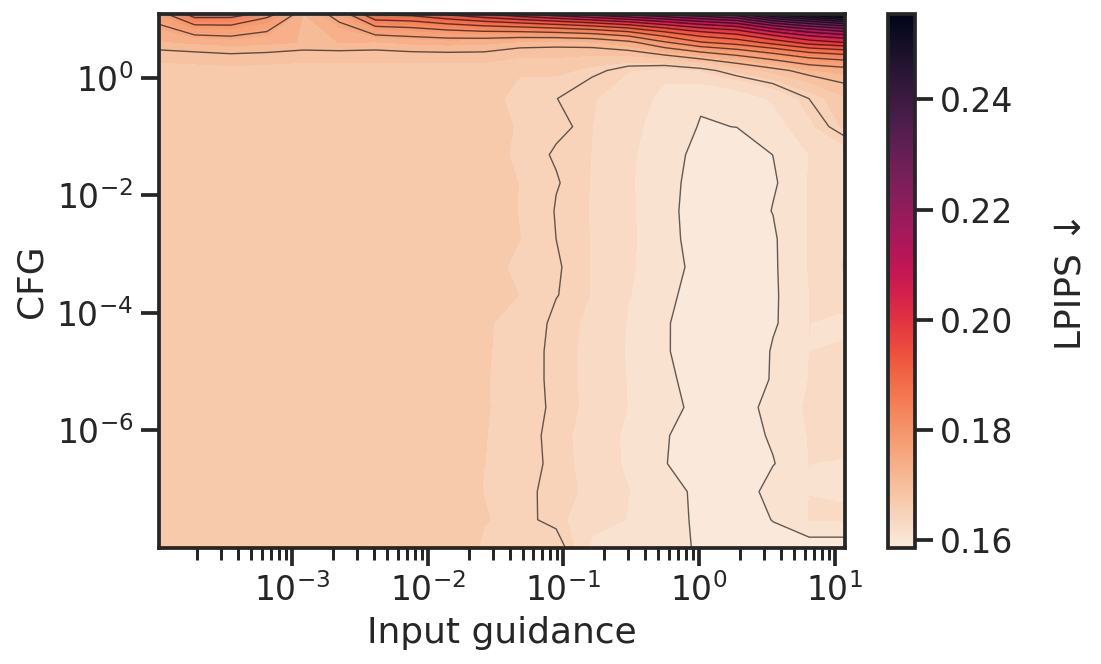}
        \end{minipage}
        \caption{AsyncPatch}
        \label{fig:asyncpatch}
    \end{subfigure}
    \vspace{0.2em}

    \begin{subfigure}[b]{\linewidth}
        \centering
        \begin{minipage}[b]{0.49\linewidth}
            \centering
            \includegraphics[width=\linewidth]{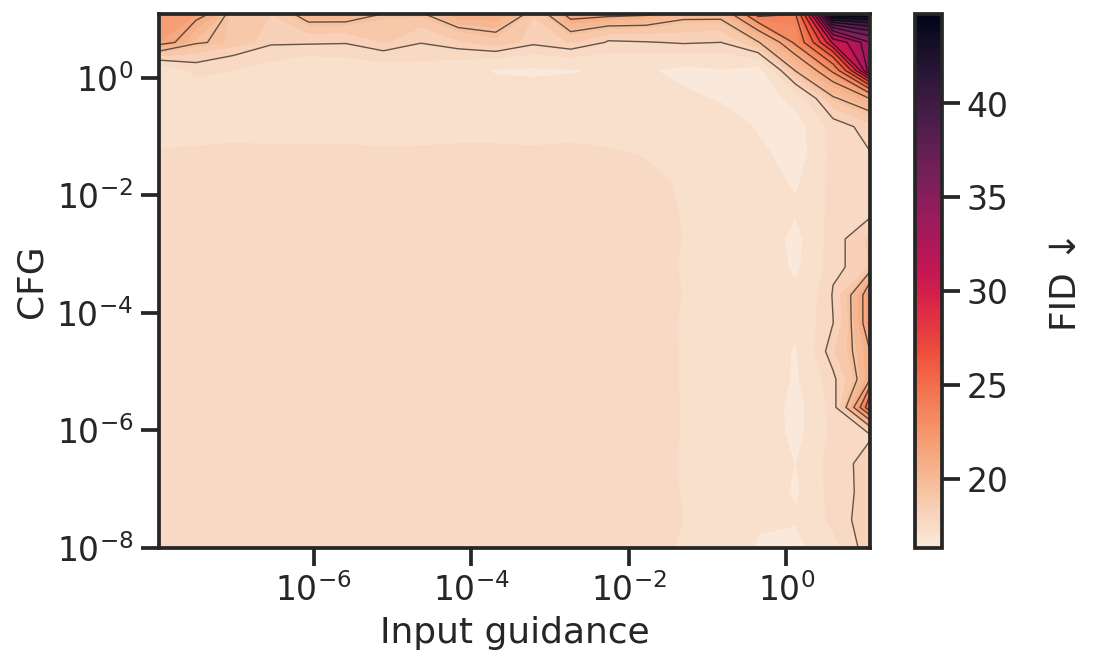}
        \end{minipage}
        \hfill
        \begin{minipage}[b]{0.49\linewidth}
            \centering
            \includegraphics[width=\linewidth]{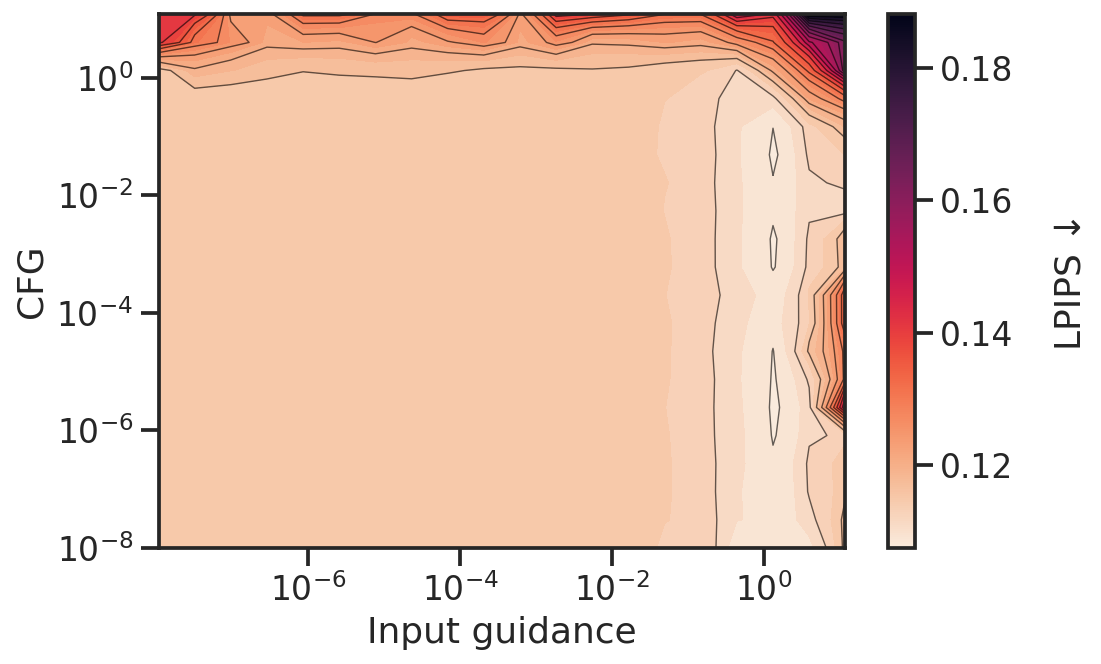}
        \end{minipage}
        \caption{Baseline}
        \label{fig:baseline}
    \end{subfigure}

    \caption{Effect of input guidance and classifier-free guidance (CFG) on both LPIPS and FID performances. Input guidance is crucial for matching details given the known region.}
    \label{fig:pareto-input-guidance}
\end{figure*}

We experiment with input guidance on both traditional diffusion models and AsyncPatch. To get input guidance with traditional diffusion models we simply set the conditioning image to $0$, which we have also done during training with $p=0.1$. For AsyncPatch, we follow the recipe laid out in the method section. Importantly, in traditional diffusion models, input guidance can only be done on the whole input, while with AsyncPatch, we can choose which part of the input to use as guidance, by simply adding more noise to that region.

\paragraph{Comparison of Input and Classifier-free guidance.} When the two guidances are combined, it becomes apparent that input guidance leads the model to pay specific attention to the texture and the details of the region of interest. This is reflected by the improvement in LPIPS, which is stronger than the improvement in FID. Instead, semantic appearance, which is better measured by the FID, sees a similar improvement with both input and classifier-free guidance. From \ref{fig:pareto-input-guidance} we can also appreciate how CFG is not enough to get the inpainted regions to more closely match the original ones, instead it can be used to better match the dataset's original distribution.

\clearpage
\section{Samples for ImageNet 64}
\label{sec:samples_imagenet_64}

\newcommand{\ImageNetSamplingRow}[1]{%
    \begin{subfigure}{0.23\linewidth}
        \centering
        \includegraphics[width=\linewidth]{paper_figures/imagenet_64/baseline_label_#1/baseline_label_#1.png}%
        \label{fig:imagenet64-label-#1-baseline}
    \end{subfigure}
    \hfill
    \begin{subfigure}{0.23\linewidth}
        \centering
        \includegraphics[width=\linewidth]{paper_figures/imagenet_64/perlin_label_#1/perlin_label_#1.png}%
        \label{fig:imagenet64-label-#1-perlin}
    \end{subfigure}
    \hfill
    \begin{subfigure}{0.23\linewidth}
        \centering
        \includegraphics[width=\linewidth]{paper_figures/imagenet_64/patchwise_label_#1/patchwise_label_#1.png}%
        \label{fig:imagenet64-label-#1-patchwise}
    \end{subfigure}
    \hfill
    \begin{subfigure}{0.23\linewidth}
        \centering
        \includegraphics[width=\linewidth]{paper_figures/imagenet_64/better_time_label_#1/better_time_label_#1.png}%
        \label{fig:imagenet64-label-#1-asyncpatch}
    \end{subfigure}

    \vspace{0.35em}%
}
\begin{figure}[htbp]
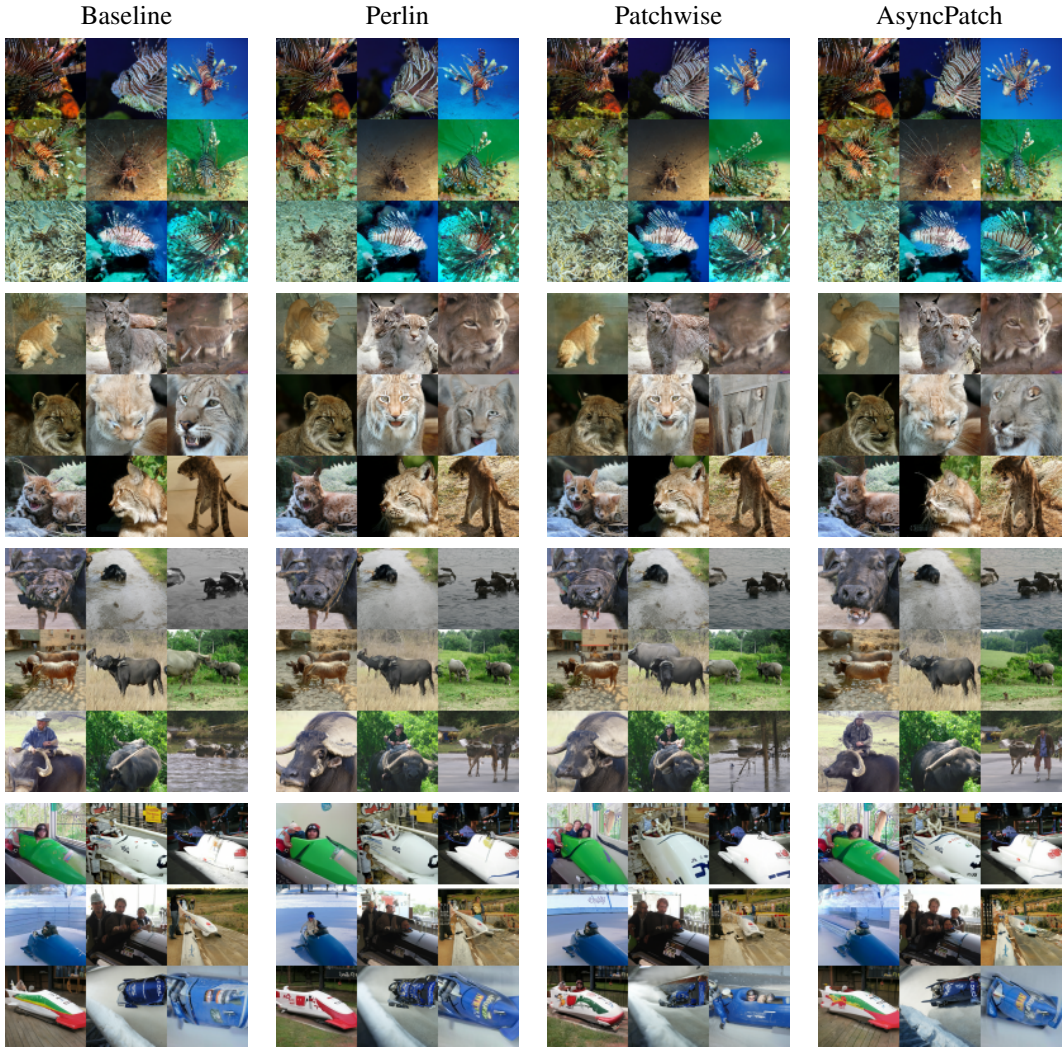

    \centering

    \makebox[0.23\linewidth]{\centering Baseline}
    \hfill
    \makebox[0.23\linewidth]{\centering Perlin}
    \hfill
    \makebox[0.23\linewidth]{\centering Patchwise}
    \hfill
    \makebox[0.23\linewidth]{\centering AsyncPatch}

    \vspace{0.25em}

    \ImageNetSamplingRow{396}
    \ImageNetSamplingRow{287}
    \ImageNetSamplingRow{346}
    \ImageNetSamplingRow{450}

    \caption{Comparison of timestep sampling methods on four ImageNet-64 classes. The same seed and sampling strategy is used for all methods. Different classes are shown in different rows.}
    \label{fig:timestep-sampling-imagenet64}
\end{figure}
\clearpage
\section{Samples for ImageNet 256 AsyncPatch}
\providecommand{\imagenetdir}{paper_figures/imagenet_256}
\providecommand{\imagenetimage}[2]{%
    \includegraphics[width=0.19\linewidth]{\imagenetdir/label_#1/id_#2_label_#1.png}%
}

\begin{figure}[!htbp]
    \centering
    {
        \setlength{\tabcolsep}{0.75pt}
        \renewcommand{\arraystretch}{0.2}
        \begin{tabular}{@{}*{5}{c}@{}}
            \imagenetimage{979}{00128} &
            \imagenetimage{979}{00129} &
            \imagenetimage{979}{00130} &
            \imagenetimage{979}{00131} &
            \imagenetimage{979}{00132} \\
            \imagenetimage{979}{00133} &
            \imagenetimage{979}{00134} &
            \imagenetimage{979}{00135} &
            \imagenetimage{979}{00136} &
            \imagenetimage{979}{00137} \\
            \imagenetimage{979}{00138} &
            \imagenetimage{979}{00139} &
            \imagenetimage{979}{00140} &
            \imagenetimage{979}{00141} &
            \imagenetimage{979}{00142} \\
            \imagenetimage{979}{00143} &
            \imagenetimage{979}{00144} &
            \imagenetimage{979}{00145} &
            \imagenetimage{979}{00146} &
            \imagenetimage{979}{00147} \\
            \imagenetimage{979}{00148} &
            \imagenetimage{979}{00149} &
            \imagenetimage{979}{00150} &
            \imagenetimage{979}{00151} &
            \imagenetimage{979}{00152}
        \end{tabular}
    }
    \caption{Generated ImageNet 256 samples using AsyncPatch latent diffusion.}
    \label{fig:imagenet_256_label_979}
\end{figure}
\clearpage

\begin{figure}[p]
    \centering
    {
        \setlength{\tabcolsep}{0.75pt}
        \renewcommand{\arraystretch}{0.2}
        \begin{tabular}{@{}*{5}{c}@{}}
            \imagenetimage{147}{00384} &
            \imagenetimage{147}{00385} &
            \imagenetimage{147}{00386} &
            \imagenetimage{147}{00387} &
            \imagenetimage{147}{00388} \\
            \imagenetimage{147}{00389} &
            \imagenetimage{147}{00390} &
            \imagenetimage{147}{00391} &
            \imagenetimage{147}{00392} &
            \imagenetimage{147}{00393} \\
            \imagenetimage{147}{00394} &
            \imagenetimage{147}{00395} &
            \imagenetimage{147}{00396} &
            \imagenetimage{147}{00397} &
            \imagenetimage{147}{00398} \\
            \imagenetimage{147}{00399} &
            \imagenetimage{147}{00400} &
            \imagenetimage{147}{00401} &
            \imagenetimage{147}{00402} &
            \imagenetimage{147}{00403} \\
            \imagenetimage{147}{00404} &
            \imagenetimage{147}{00405} &
            \imagenetimage{147}{00406} &
            \imagenetimage{147}{00407} &
            \imagenetimage{147}{00408}
        \end{tabular}
    }
    \caption{Generated ImageNet 256 samples using AsyncPatch latent diffusion.}
    \label{fig:imagenet_256_label_147}
\end{figure}
\clearpage

\begin{figure}[p]
    \centering
    {
        \setlength{\tabcolsep}{0.75pt}
        \renewcommand{\arraystretch}{0.2}
        \begin{tabular}{@{}*{5}{c}@{}}
            \imagenetimage{660}{04736} &
            \imagenetimage{660}{04737} &
            \imagenetimage{660}{04738} &
            \imagenetimage{660}{04739} &
            \imagenetimage{660}{04740} \\
            \imagenetimage{660}{04741} &
            \imagenetimage{660}{04742} &
            \imagenetimage{660}{04743} &
            \imagenetimage{660}{04744} &
            \imagenetimage{660}{04745} \\
            \imagenetimage{660}{04746} &
            \imagenetimage{660}{04747} &
            \imagenetimage{660}{04748} &
            \imagenetimage{660}{04749} &
            \imagenetimage{660}{04750} \\
            \imagenetimage{660}{04751} &
            \imagenetimage{660}{04752} &
            \imagenetimage{660}{04753} &
            \imagenetimage{660}{04754} &
            \imagenetimage{660}{04755} \\
            \imagenetimage{660}{04756} &
            \imagenetimage{660}{04757} &
            \imagenetimage{660}{04758} &
            \imagenetimage{660}{04759} &
            \imagenetimage{660}{04760}
        \end{tabular}
    }
    \caption{Generated ImageNet 256 samples using AsyncPatch latent diffusion.}
    \label{fig:imagenet_256_label_660}
\end{figure}

\clearpage
\section{Texture synthesis}
\label{sec:app_texture_synthesis}
\begin{figure}[hbtp]

    \centering
    \adjincludegraphics[width=0.8\linewidth]{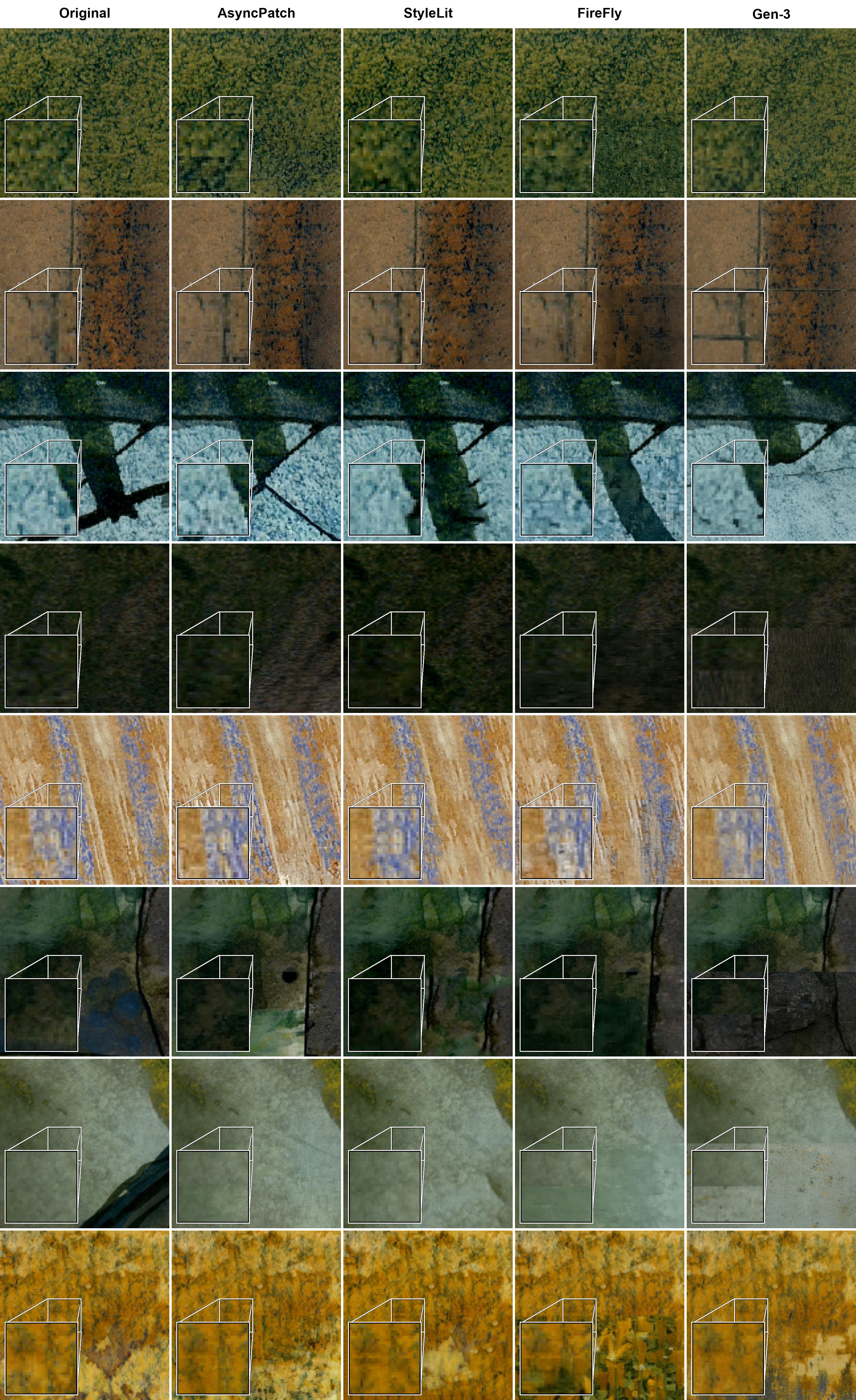}
    \caption{Qualitative comparison of texture synthesis. Images are original and un-altered, zoom-in box is added with nearest-neighbour interpolation for better visualization.}
    \label{fig:stylization_appendix}
\end{figure}

\clearpage